\documentclass{article}

\usepackage{microtype}
\usepackage{graphicx}
\usepackage{subfigure}
\usepackage{booktabs} 
\usepackage{graphicx}
\usepackage{wrapfig}

\usepackage[table]{xcolor}
\usepackage{hyperref}

\newcommand{\revision}[1]{{\color{black} #1}}

\usepackage[accepted]{icml2025}

\usepackage{amsmath}
\usepackage{amssymb}
\usepackage{mathtools}
\usepackage{amsthm}

\usepackage[capitalize,noabbrev]{cleveref}

\theoremstyle{plain}
\newtheorem{theorem}{Theorem}[section]

\theoremstyle{definition}

\theoremstyle{remark}
\newtheorem{remark}[theorem]{Remark}
\usepackage{multirow}
\usepackage{colortbl}
\usepackage{booktabs}
\usepackage{hhline}

\usepackage[textsize=tiny]{todonotes}

\usepackage{bbm}
\icmltitlerunning{Speeding up Speculative Decoding via Sequential Approximate Verification}

\begin{document}

\twocolumn[
\icmltitle{Speeding up Speculative Decoding via Sequential Approximate Verification}



\icmlsetsymbol{equal}{*}

\begin{icmlauthorlist}
\icmlauthor{Meiyu Zhong}{equal,yyy}
\icmlauthor{Noel Teku}{equal,yyy}
\icmlauthor{Ravi Tandon}{yyy}

\end{icmlauthorlist}

\icmlaffiliation{yyy}{Department of Electrical and Computer Engineering, University of Arizona, Tucson, US}

\icmlcorrespondingauthor{Meiyu Zhong}{meiyuzhong@arizona.edu}
\icmlcorrespondingauthor{Noel Teku}{nteku1@arizona.edu}
\icmlcorrespondingauthor{Ravi Tandon}{tandonr@arizona.edu}

\icmlkeywords{Machine Learning, ICML}

\vskip 0.3in
]


\printAffiliationsAndNotice{\icmlEqualContribution}

\begin{abstract}

Speculative Decoding (SD) is a recently proposed technique for faster inference using Large Language Models (LLMs). SD operates by using a smaller draft LLM for autoregressively generating a sequence of tokens and a larger target LLM for parallel verification to ensure statistical consistency.  However, periodic parallel calls to the target LLM for verification prevent SD from achieving even lower latencies. We propose \textit{SPRINTER}, which utilizes a low-complexity verifier trained to predict if tokens generated from a draft LLM would be accepted by the target LLM. By performing \textit{sequential approximate  verification}, \textit{SPRINTER} does not require verification by the target LLM and is only invoked when a token is deemed unacceptable. This reduces the number of calls to the larger LLM, achieving further speedups and lower computation cost. We present a theoretical analysis of \textit{SPRINTER}, examining the statistical properties of the generated tokens, as well as the expected reduction in latency as a function of the verifier. We evaluate \textit{SPRINTER} on several datasets and model pairs, demonstrating that approximate verification can still maintain high quality generation while further reducing latency. 

\end{abstract}

\section{Introduction}
\label{sec:Intro}

\begin{figure}
    \centering
    \includegraphics[width=0.99\linewidth]{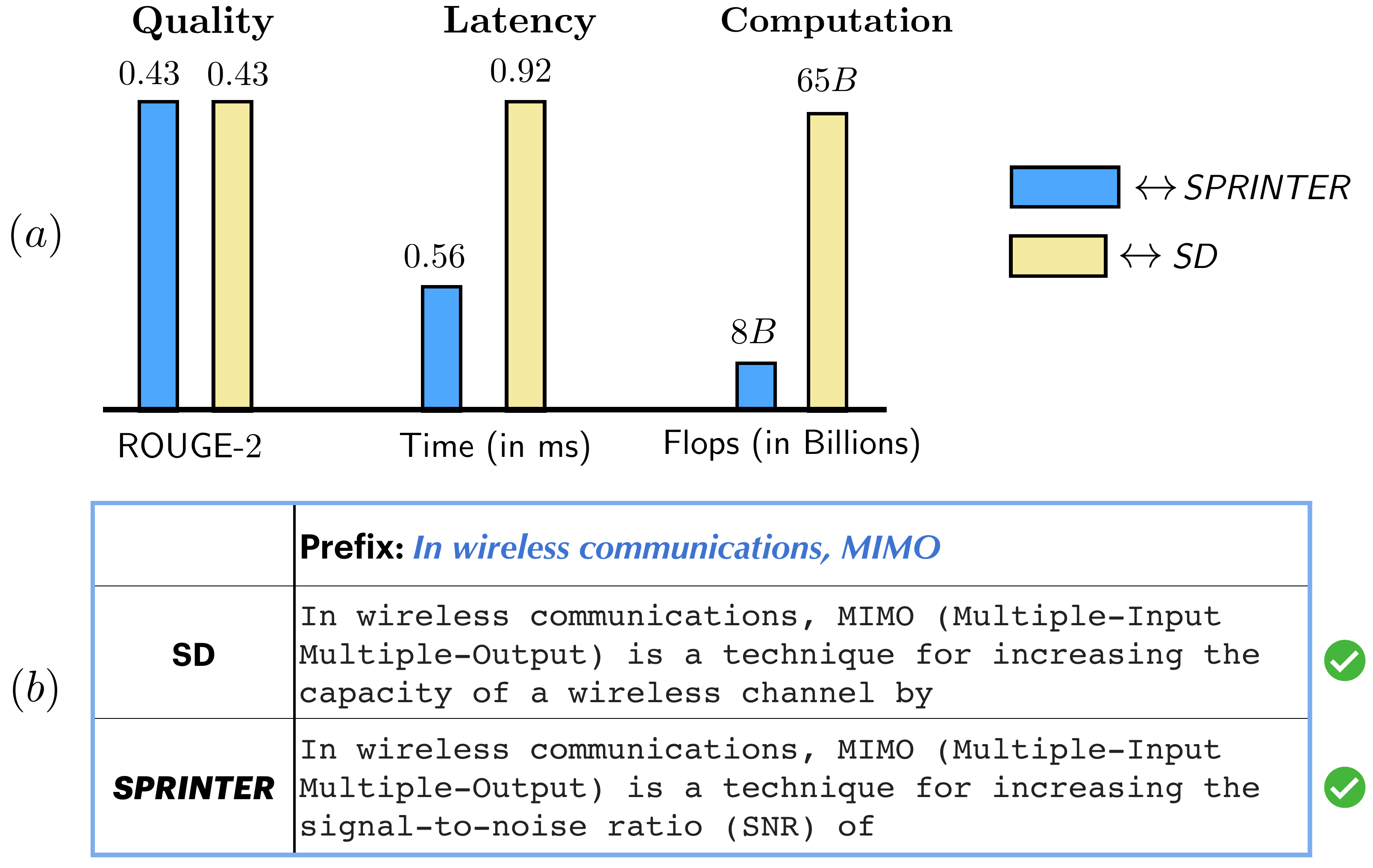}\vspace{-8pt}
    \caption{ \revision{(a) Comparison between \textit{SPRINTER} and SD with respect to \textit{Quality} (ROUGE score), \textit{Latency} (time in ms required to generate a token) and \textit{Computation} (number of flops required to generate $20$ consecutive acceptable tokens from the draft model). \textit{SPRINTER} can attain comparable quality, $\mathbf{1.64X}$ speedups, and $\mathbf{8X}$ smaller computation costs compared to SD. (b) Example responses generated via  \textit{SPRINTER} vs SD given the prefix: ``In wireless communications, MIMO". The response from \textit{SPRINTER} is comparable to SD (for more examples, see Section \ref{sec:prompts_ex}).}
    }
    \label{fig: quality_latency_comput}
    \vspace{-10pt}
\end{figure}
\begin{figure*}[t]
    \centering
    \includegraphics[width=0.9\linewidth]{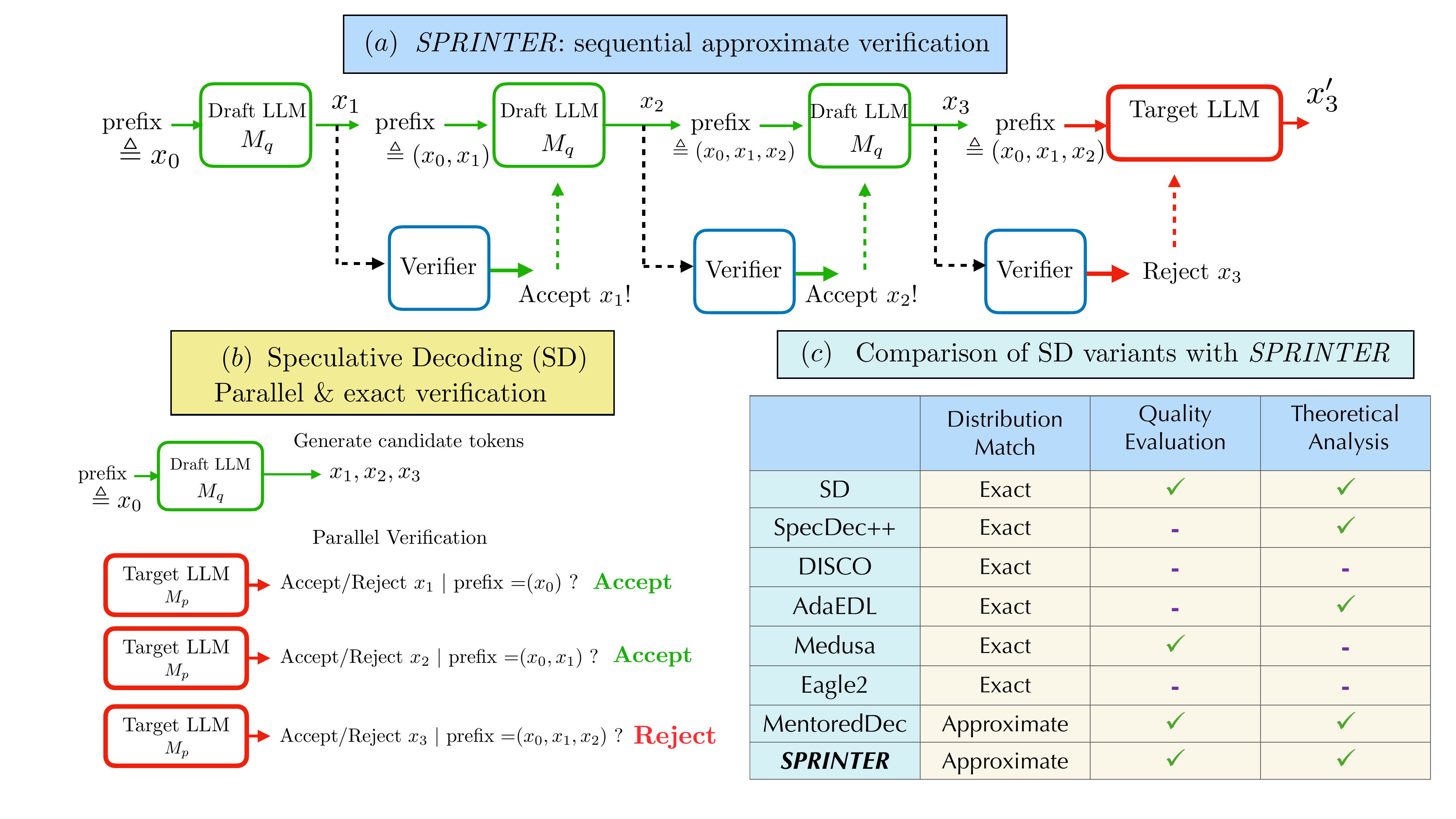}
    \caption{\revision{(a) \textit{SPRINTER} works by generating tokens from a smaller (draft) LLM, which are sequentially accepted/rejected by a \textit{verifier}, a low-complexity small classifier. In \textit{SPRINTER}, the larger (target) LLM is only called if a token is rejected and used only to replace the rejected token. (b) Speculative decoding (SD) works by generating multiple tokens by the draft model, all of them are verified in parallel by the target LLM. (c)
    Comparison of different speculative decoding based mechanisms with respect to three aspects: i) Approximate or Exact match with the larger LLM (target) probability distribution, ii) Quality Evaluation of Completions and iii) Theoretical Analysis. \textit{SPRINTER} is the first framework to provide in-depth analysis of sequential approximate verification; its impact on the quality-vs-latency tradeoff and provide insights on how to navigate this tradeoff. Additional discussion on related works is presented in Section \ref{sec:relatedworks_append}.}}
    \label{fig:frameworkSPrinter}
\end{figure*}

Large Langauge Models (LLMs) have shown to be very effective in different applications including text generation \cite{zingale2024language}, image analysis \cite{niu2024text}, and video understanding \cite{tang2024videounderstandinglargelanguage}. However, despite this success, LLM-based solutions for various problem domains are still constrained by \textit{high computational costs incurred during inference} due to large model sizes. To enable faster inference, Speculative Decoding (SD) \cite{leviathan2023fast} has been proposed as a solution for reducing the latencies incurred during the inference of significantly larger LLMs. Under this paradigm, a smaller draft LLM is used to autoregressively generate a certain number of tokens. These tokens are then passed to a larger target LLM, which processes the tokens in parallel to determine how many of them are acceptable (i.e. if the distribution of the generated tokens matches the distribution of what the target model would have generated). If they are not acceptable, then the target model is called to generate replacement tokens. By reducing the amount of times the target model is invoked for autoregressive generation, less latency is incurred. Thus, the objective of SD is to incur smaller inference times while guaranteeing that sampling tokens from the draft model is equivalent to sampling from the target model. \cite{mamou2024dynamicspeculationlookaheadaccelerates} proposed using a two layer feedforward network to stop the draft model from generating tokens and initiate the target model's verification process, once the output of the network is less than a threshold. 
\cite{huang2024specdec++} models the SD procedure as a Markov Decision Process and uses the draft model with a smaller, multi-layer network to predict the probability that the current token generated by the draft model should be accepted. The probability that there is at least one draft token that should be rejected is subsequently derived and compared with a threshold to determine if the target model should be invoked for verification.

\textbf{Overview of SPRINTER:} Running the target model for parallel verification, even periodically (i.e. after a certain number of tokens) can still result in significant latencies. Higher speedups can be attained if the constraint that the tokens generated by the draft model must match the distribution of the target model is relaxed. However, we do not want to deviate too far from the target distribution as this would increase the likelihood of the draft model generating inaccurate tokens. To balance this tradeoff, we propose \textit{SPRINTER}, a sampling technique that uses a low-complexity verifier that predicts when it is necessary to invoke the target model to generate a token that replaces the current draft token. Subsequently, under \textit{SPRINTER}, verification is performed \textit{sequentially} as draft tokens are generated, in contrast to SD based approaches which perform parallel verification of multiple draft tokens through the larger LLM, which requires more computational resources to execute.  

In many computational settings, including machine learning and optimization, the cost of generating a valid solution can be significantly higher than that of verifying one. This asymmetry, often referred to as the \textit{generation-verification gap}, also provides inspiration for \textit{SPRINTER}: verifying the acceptability of a token could be  easier than generating a high quality token from a larger LLM. The verifier used in this work has a significantly lower complexity compared to the draft and target LLMs; subsequently, the additional overhead introduced by training and using it is negligible (as further evidenced in the results). b) Our approach is also aligned with observations made in recent works such as \cite{judge,melcer2024approximately}, which have shown that \textit{high-quality} tokens can still be generated without necessarily matching the distribution of the larger target model. \revision{Figure \ref{fig: quality_latency_comput}(a) shows that \textit{SPRINTER} achieves lower latency and computational cost while maintaining quality comparable to that of SD. Figure \ref{fig: quality_latency_comput}(b) presents example responses generated by \textit{SPRINTER} and SD, illustrating that \textit{SPRINTER} produces responses of similar quality as SD.} We next summarize the main contributions of this paper.
\begin{itemize}
\vspace{-5pt}
    \item \textbf{\textit{SPRINTER} framework}. We propose \textit{SPRINTER}, a framework for achieving faster inference from LLMs using a pair of (draft (small), target (large)) LLMs together with the aid of a verifier. The role of the verifier is to perform approximate verification, i.e., if tokens generated by a draft model would be acceptable by the larger target LLM. The key motivating factors and the detailed framework is described in Section \ref{sec:SPRINTER}. 
\vspace{-5pt}
    \item \textbf{Theoretical Analysis}. We present a comprehensive theoretical analysis in Section \ref{sec:theory} to show the tradeoffs between quality of generated tokens versus latency speedups and computational savings offered by \textit{SPRINTER}. Specifically, we demonstrate how the ROC curve characteristics (e.g., false-positive and true-positive rates) of the verifier can be used to balance the tradeoff between latency and quality. Furthermore, we discuss strategies to train the verifier and show how the theoretical results also provide useful design insights. 
\vspace{-1pt}
    \item \textbf{Experiments and Validation}.
    \revision{We present a comprehensive evaluation of \textit{SPRINTER} on several datasets and model pairs in Section \ref{sec:results}, demonstrating its ability to reduce latency while requiring significantly less computation and maintaining high quality. Specifically, Win-tie rate and ROUGE metrics are used to evaluate the quality of responses generated by \textit{SPRINTER} against those generated with SD, indicating that only minimal quality degradation is experienced. \textit{SPRINTER} is also shown to outperform target distribution-preserving SD variants (e.g. Eagle2 \cite{li2406eagle}, Medusa \cite{cai2024medusa}) in speed and quality. Furthermore, higher performance improvements are attained using \textit{SPRINTER} compared to Mentored Decoding \cite{tranthien2024mentored}, which similarly relaxes the requirement that generated tokens match the target distribution.}
\end{itemize}



\vspace{-10pt}
\section{Preliminaries on Speculative Decoding (SD)}
\label{sec:prelims}
SD was originally proposed in \cite{leviathan2023fast} as a novel algorithm for speeding up LLM inference from a larger target LLM $M_p$ 
through the help of a smaller (faster) draft LLM $M_q$. $p(x)$ and $q(x)$ represent the probability distributions we get from $M_p$ and $M_q$ respectively for the next token given a specific prefix (i.e. set of already generated tokens $x_{<t}$); specifically, we let $p(x)$ and $q(x)$ denote $p(x|\text{prefix})$ and $q(x|\text{prefix})$ respectively. First, given a prefix, the draft model $M_q$ autoregressively generates $\gamma$ tokens, where $\gamma$ is a hyperparameter chosen by the user. These sequence of tokens are then passed to the target model which performs verification of all $\gamma$ candidate completion sequences in parallel. For the last token $x$ in each sequence, the target model $M_p$ is invoked and it verifies the relationship between $q(x)$ and $p(x)$ (see Figure \ref{fig:frameworkSPrinter} for an example). If $q(x) < p(x)$, meaning that the probability of the draft token is within the distribution of the target token, it is acceptable. However, if a draft token in any one of the candidate sequences is not accepted, it can still be accepted with probability $\frac{p(x)}{q(x)}$; otherwise, the target model resamples from an alternative distribution given as follows \cite{leviathan2023fast}: 
\begin{align}
  p'(x) = \text{norm}(\max(0, p(x) - q(x))).
    \label{eq:resample}
\end{align}
SD ensures that the entire sampling process matches the distribution of the target LLM. Speculative decoding has turned out to be powerful in achieving speedups and has led to a large number of recent papers that have proposed variations of SD, including works shown in Figure \ref{fig:frameworkSPrinter}(c). \cite{mamou2024dynamicspeculationlookaheadaccelerates} and \cite{huang2024specdec++}, for example, use a low-complexity classifier to determine when the draft model should stop generating tokens while ensuring that the sampled tokens match the distribution of the target model. \cite{kim2024speculative} and \cite{agrawal2024adaedl} propose heuristics that incur less complexity compared to using a verifier; however, they are also ensure that the sampled tokens match the distribution of the target model. We provide a more detailed discussion on these related works in Section \ref{sec:relatedworks_append}.

\begin{figure}
    \centering
    \includegraphics[width=0.7\linewidth]{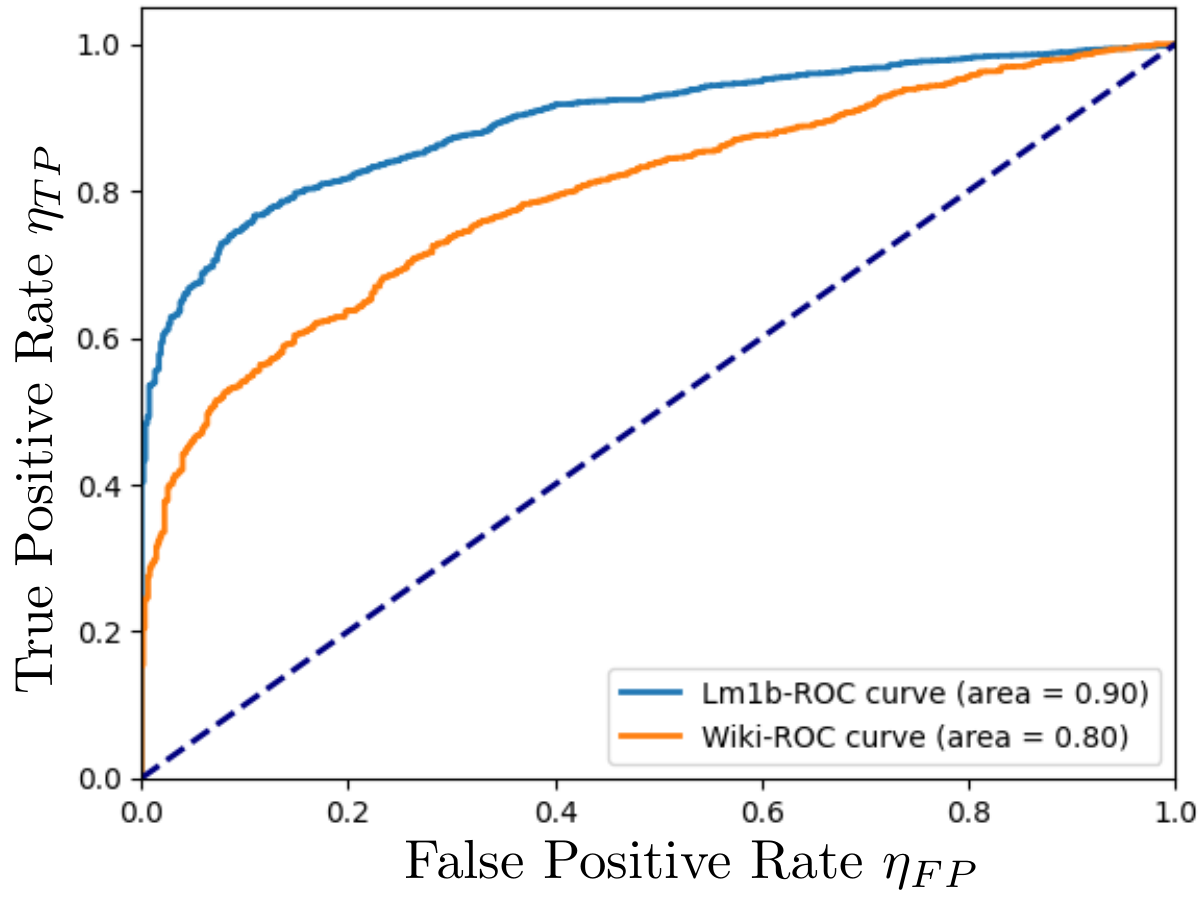}\vspace{-5pt}
    \caption{ROC Curve Performance of a trained Verifier (less than $1k$ parameters) on the Wiki-Summary and LM1B Datasets with GPT-Neo-125M as the draft model $M_q$ and  GPT-Neo-1.3B as the target model $M_p$. Despite being orders of magnitude smaller in size compared to the draft and target models, the verifier was able to achieve AU-ROC of $0.8$ (respectively $0.9$) on the two datasets.}
    \label{fig:ROC_wiki_LM1B_12}
    \vspace{-10pt}
\end{figure}



\section{SPRINTER Framework \& Analysis} 
\label{sec:SPRINTER}
Before presenting our framework, we first discuss some of the key motivating factors behind \textit{SPRINTER}.

\textit{\textbf{Cost of Parallelism}}. As SD and several variants discussed in related work attempt to match the target distribution, they end up invoking the target model which performs parallel verification. While parallelism ensures that the latency (time) for verifying $\gamma$ tokens is equivalent to running the target model once, one still has to pay the cost of parallelism as the target model runs $\gamma$ times. As $\gamma$ increases, the latency reduces but the cost of parallelism grows proportionally. This is the first idea that motivates us to study sequential verification; instead of verifying by the target LLM in parallel, we instead propose \textit{approximate verification} by a significantly smaller model (named the verifier) in a sequential manner as shown in Fig. \ref{fig:frameworkSPrinter}(a). Depending on the quality of the verifier (i.e., false-positive and true-positive rates), we only call the target LLM if a token is rejected by verifier. Thus, \textit{SPRINTER} can achieve the dual benefit of reducing the number of calls to the target LLM and completely eliminates the cost of running it in parallel. 

As an illustration, Fig. \ref{fig:ROC_wiki_LM1B_12} shows the ROC curve of a low-complexity verifier (with a single layer and less than $1k$ training parameters) which was trained to accept/reject tokens generated by GPT-Neo-125M (draft model $M_q$) and  GPT-Neo-1.3B as the target model $M_p$. The fact that we were able to achieve AU-ROC (area under ROC curve) of $0.8$ and $0.9$ was achievable on LM1B and Wiki-summaries datasets first highlights the feasibility of low-complexity verification (more results are presented in Section \ref{sec:results}).

\textit{\textbf{Quality of Smaller Models}}. \revision{Inevitably, if we resort to approximate verification, we have to give up statistical consistency, i.e., one cannot guarantee a match with the target LLM distribution. However, statistical consistency alone may not be a necessary indicator for high quality generation. For instance, recent works such as Mentored Decoding \cite{tranthien2024mentored} and Judge Decoding \cite{judge} have shown that even smaller LLMs have generation capabilities that can be comparable with larger ones. 
Our idea behind \textit{SPRINTER} is to use a smaller trained model (verifier) which is pipelined with the smaller model for sequential verification. It is this interplay between latency, total computational costs and quality that motivate \textit{SPRINTER}. We next describe the framework in detail followed by the theoretical analysis of \textit{SPRINTER}. }


\textit{\textbf{Algorithm}}. We now provide an overview of the \textit{SPRINTER} sampling process. First, a prefix is fed to the draft model and a token is sampled with probability $q(x)$. The draft token is then passed to a verifier $V$ to predict whether or not the ratio between $q(x)$ and $p(x)$ is greater or less than 1. The verifier can take as input various latent features derived from the draft model (e.g. embedding of the draft token $x$, probability distribution of the LLM's vocabulary). As there is flexibility with which features from the draft model can be extracted, we denote the input to the verifier as $s(x,\text{prefix})$. An ideal verifier would make the following decision: 
\[
{V(s(x,\text{prefix}))} =
\begin{cases}
1   & \frac{q(x)}{p(x)} \leq 1  \\

0& \frac{q(x)}{p(x)} > 1. \\
\end{cases}
\]

\begin{figure*}[t]
    \centering
    \includegraphics[width=0.88\linewidth]{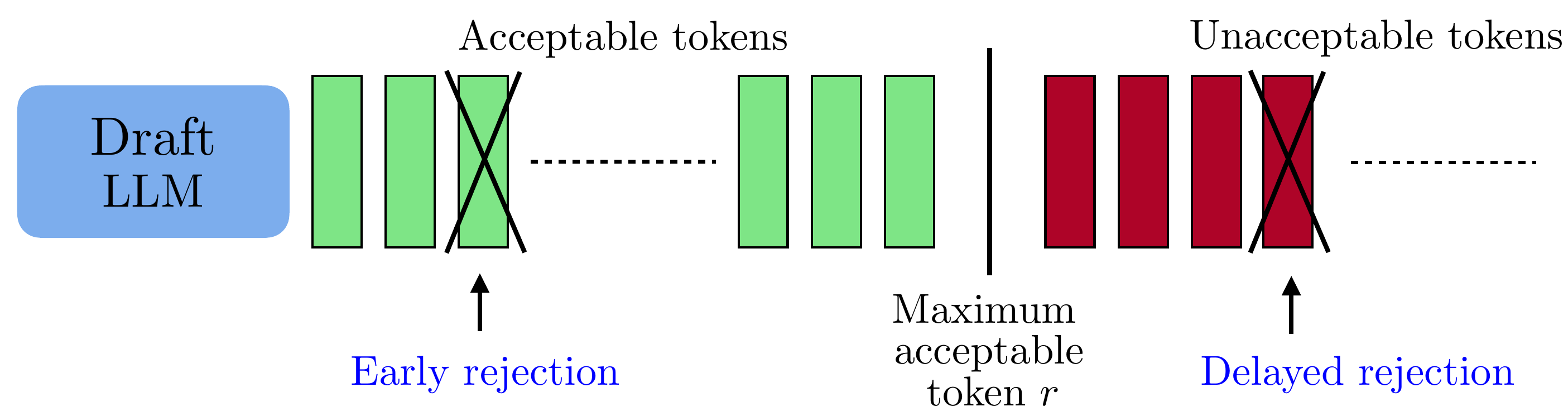} \vspace{-4pt}
    \caption{A draft LLM generates $r$ tokens that are acceptable and subsequent tokens that are unacceptable. If the verifier rejects one of the first $r$ consecutive tokens, the speedup attained from continuing to accept tokens until the $r^{th}$ token is lost (i.e. early rejection). If the verifier continues to accept tokens after the $r^{th}$ token, it experiences smaller latencies but at the cost of accepting low-quality tokens (i.e. delayed rejection). Theorem \ref{the:SPRINTER_expected_tokens} characterizes the expected number of generated tokens as a function of \(\eta_\text{TP}\) and \(\eta_\text{FP}\).}
    \label{fig:token_generation}
\end{figure*}
 If $V(s(x,\text{prefix})) = 1$, the verifier predicts that the ratio is less than 1, suggesting that the current token $x$ is acceptable and that $M_q$ should generate the next token. If $V(s(x,\text{prefix}))=0$, the verifier predicts that the ratio is larger than 1, leading to the rejection of the current token and indicating that $M_p$ should be called. In this scenario, similar to SD, the draft token can be accepted with probability $\frac{p(x)}{q(x)}$ or rejected with probability $1-\frac{p(x)}{q(x)}$ and replaced with a token sampled from the revised distribution \eqref{eq:resample}. In Section \ref{sec:learn_veri}, we provide details on how to train a verifier.  We illustrate the sampling process of \textit{SPRINTER} in Algorithm \ref{alg:alg_short} (full algorithm is presented in Section \ref{sec:algorithm_hyper}) and show the \textit{SPRINTER} sampling process in Fig \ref{fig:frameworkSPrinter}(a).

\begin{algorithm}[h]
   \caption{\textit{SPRINTER}}
   \label{alg:alg_short}
\begin{algorithmic}
   \STATE {\bfseries Input:} $M_p, M_q, V, \text{Prefix}, \text{Prediction Threshold}~\tau$
   \STATE Initiate the values
   \WHILE{True}
   \STATE Update the Prefix
    \STATE Generate the current token $x$ from $M_q$
    \STATE Obtain the Verifier's prediction of the current token $V (s(x, \text{Prefix}))$ . 
      
    \IF{$V (s(x, \text{Prefix})) \leq \tau$}
    \STATE Break
    \ENDIF
    \ENDWHILE
   \STATE Invoke $M_p$ to verify the \textcolor{blue}{last token} and re-sample if necessary.
 
\end{algorithmic}
\end{algorithm}
\subsection{Theoretical Analysis of SPRINTER}\label{sec:theory}
In this Section, we present our theoretical results on \textit{SPRINTER}. Through these results, we aim to study the impact of verifier's performance on a) the probability distribution of tokens sampled by \textit{SPRINTER}; b) expected number of consecutive tokens sampled by \textit{SPRINTER} before invoking the target model; c) average latency incurred in the process as well as the amount of computational savings. \revision{All theoretical results in our paper are derived under the assumption that the verifier operates in an i.i.d. manner. }
 As we discuss later in this Section, these results also provide practical design insights for navigating the quality-vs-latency tradeoffs.  

\textbf{(a) Statistical Analysis of generated tokens.} Let us denote $\eta_{FP}$ and $\eta_{TP}$ as the false-positive and true-positive rates of the verifier, respectively. Specifically, a false positive refers to the setting if an unacceptable token (i.e., $q(x)/p(x)>1$) is deemed acceptable by the verifier. Conversely, a true positive refers to the scenario if an acceptable token (i.e., $q(x)/p(x)\leq 1$) is accepted by the verifier. 
In our first result, we characterize the distribution of the tokens generated by \textit{SPRINTER}. Proof of Theorem \ref{the:SPRINTER} can be found in the Appendix \ref{sec:31proof}.
\begin{theorem}\label{the:SPRINTER}
    The probability of a token $x$ being chosen when running SPRINTER is given as
    \begin{align}
    \label{eq:SPRINT_DIST}
        p_\text{SPRINTER}(x) = (1-\eta_\text{FP})p(x) + \eta_\text{FP} q(x),
    \end{align}
    where $\eta_\text{FP}$ is the false positive rate of the verifier. Furthermore, the total-variation distance between the target and \textit{SPRINTER} distributions is $d_\text{TV}(p,p_\text{SPRINTER}) = \eta_\text{FP}d_\text{TV}(p,q)$.
\end{theorem}


\begin{remark}
The false positive rate $\eta_\text{FP}$ directly influences the distribution of \textit{SPRINTER} and represents how much the distribution of \textit{SPRINTER} deviates from the distribution of the target model. This observation aligns with the intuition that a perfect verifier with $\eta_\text{FP} =0$ (i.e. \textit{SPRINTER} never samples a token that should be rejected), for example, this results in $p_\text{SPRINTER} = p(x)$, effectively matching the distribution of the target model. 
\end{remark}

\textbf{(b) Expected number of generated tokens}. We now analyze the expected number of tokens accepted by the verifier when using \textit{SPRINTER} as a function of $\eta_\text{FP}$ and $\eta_\text{TP}$. For this analysis, we consider the scenario illustrated in Figure \ref{fig:token_generation}. Suppose that the ground truth is that given a prefix, the draft model $M_q$ is capable of producing $r$ acceptable tokens sequentially (in other words, the first $r$ generated tokens by $M_q$ are acceptable, whereas subsequent ones are unacceptable). Under this ground truth, let us define the random variable $N_{\text{SPRINTER}}$ as the number of consecutive tokens accepted by the verifier. Assuming that the verifier makes decisions in an i.i.d. manner, it can exhibit two types of behavior also shown in Fig. \ref{fig:token_generation}:
\begin{itemize}
    \item \textbf{Early Rejection}: This occurs if the verifier accepts the first $(i-1)$ tokens but mistakenly predicts that the $i^{th}$ token (for $i \leq r$) should be rejected, then the verifier will revert to calling the target model rather than continuing to enable the draft model to generate the remaining $r-i$ acceptable tokens. In doing so, the verifier misses out on experiencing an even greater latency reduction while still generating high-quality tokens. This indicates that $\eta_\text{TP}$ directly influences the early rejection caused by \textit{SPRINTER}.
    \item \textbf{Delayed Rejection}: On the other hand, it can happen that the verifier continues to accept more than $r$ tokens. Specifically, if the verifier first stops at the $i^{th}$ token (for $i > r$), then it does not invoke the target model until token $i$, resulting in a higher computational savings and latency speedups but at the cost of accepting $(i-r)$ lower quality tokens. This indicates that $\eta_\text{FP}$ directly influences the deviation of the statistical distribution from the target model.
\end{itemize}
Our next Theorem characterizes the properties of $N_\text{SPRINTER}$ assuming that $r$ consecutive draft tokens are acceptable. 
\begin{figure}[t]
    \centering
    \includegraphics[width=0.9\linewidth]{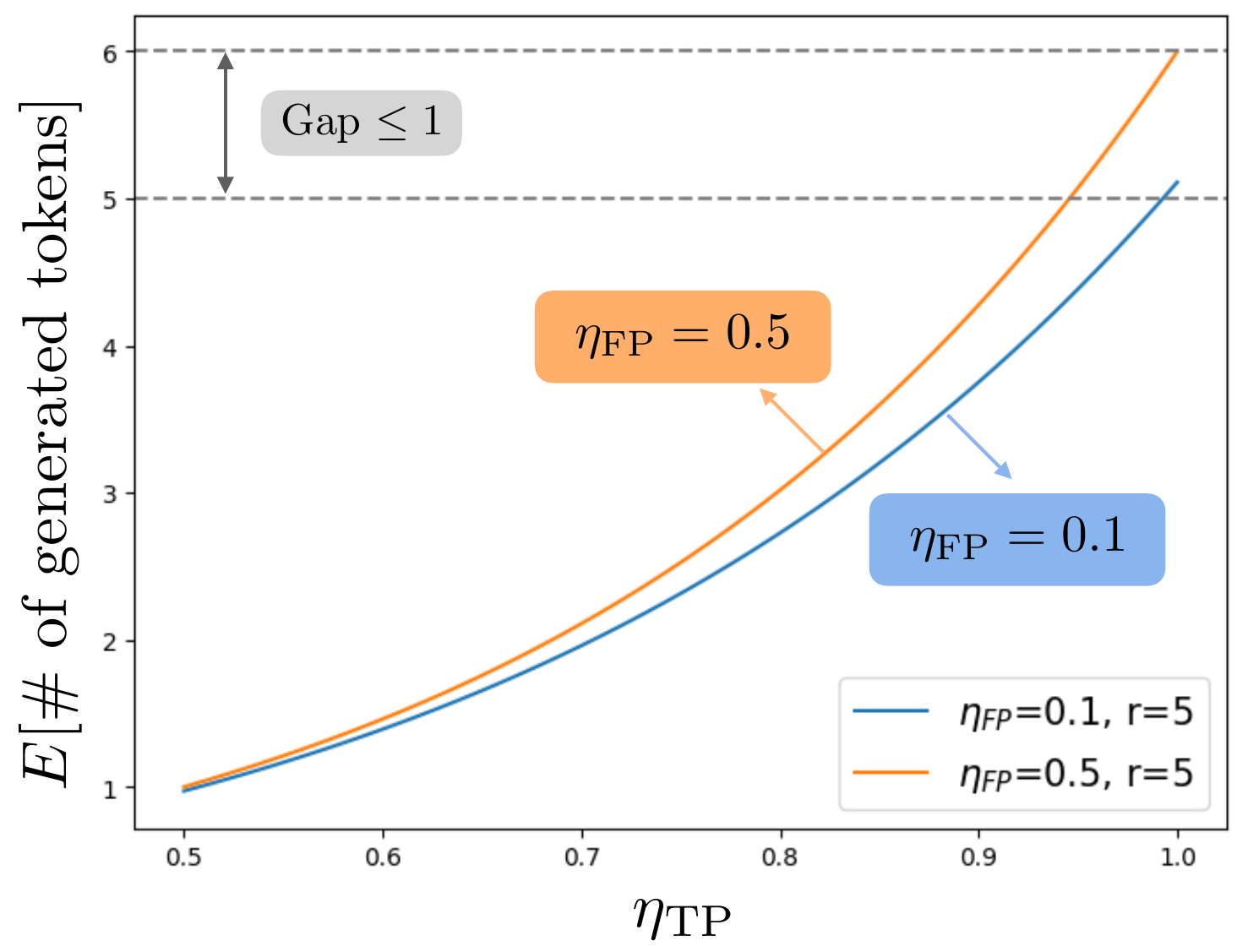}\vspace{-10pt}
    \caption{Illustration of expected number of tokens generated by \textit{SPRINTER} as a function of true-positive rate ($\eta_{TP}$) for two different values of $\eta_{FP}$ when the number of consecutively acceptable tokens is $r=5$. We can observe that as long as $\eta_{FP}\leq 0.5$, the average number of unacceptable tokens  (shown as the \textit{``Gap"}) generated by \textit{SPRINTER} never exceeds $1$.}
    \label{fig:token_tp_tradoff}
    \vspace{-15pt}
\end{figure}

\begin{theorem}\label{the:SPRINTER_expected_tokens}
    The probability distribution of the number of generated tokens is given as: 
    \[
\mathbb{P}(N_\text{SPRINTER}=i) =
\begin{cases}
\eta_\text{TP}^{i} (1-\eta_\text{TP}) & i < r, \\
\eta_\text{TP}^r(\eta_\text{FP})^{i-r}(1-\eta_\text{FP}) & i\geq r.
\end{cases}
\]
The expected number of generated tokens is given as:

    \begin{align}
\label{eq:SPRINT_EXP_TOKEN}
        \mathbb{E}(N_\text{SPRINTER})=       \frac{\eta_\text{TP}-\eta_\text{TP}^r}{1-\eta_\text{TP}} + \frac{\eta_\text{TP}^r}{1-\eta_\text{FP}}.
        \end{align}
\end{theorem}

The proof of Theorem \ref{the:SPRINTER_expected_tokens} can be found in the Appendix \ref{sec:expec_proof}. To gain more insights from this result, Fig. \ref{fig:token_tp_tradoff} illustrates the trade-off between $\eta_\text{TP}$ and $\mathbb{E}(N_\text{SPRINTER})$ for two different values of $\eta_\text{FP}$ assuming that $r = 5$. The figure indicates that as $\eta_\text{TP}$ approaches 1,  even with a verifier that has a substantial $\eta_\text{FP}$ (i.e. 0.5), the additional number of tokens that are accepted past $r$ is marginal. This implies that if an estimate of $r$ could be attained from the training data, then fixing $\eta_\text{FP}$ and varying $\eta_\text{TP}$, for example, could bring insights into how well the verifier must perform to generate close to the ideal $r$ tokens. Essentially, $\eta_\text{TP}$ and $\eta_\text{FP}$ can be varied to determine the optimal false positive and negative rates such that the chances of early and delayed rejection occurring potentially caused by the verifier are minimized. 

\textbf{(c) Latency Analysis and Computational Cost}. 
Given the result in \ref{the:SPRINTER_expected_tokens}, we now derive the latency incurred by \textit{SPRINTER} under the scenario in Figure \ref{fig:token_generation}. Let $t_d$, $t_t$ and $t_v$ represent the time required to inference the draft, target and verifier models respectively. We also assume  $t_v \leq t_d$, i.e., the time it takes to run the verifier is smaller than running the draft model. Under the above assumption, the next result characterizes the expected stopping time (i.e., the average time before the first rejection by the verifier). 

\begin{figure}
    \centering
    \includegraphics[width=0.9\linewidth]{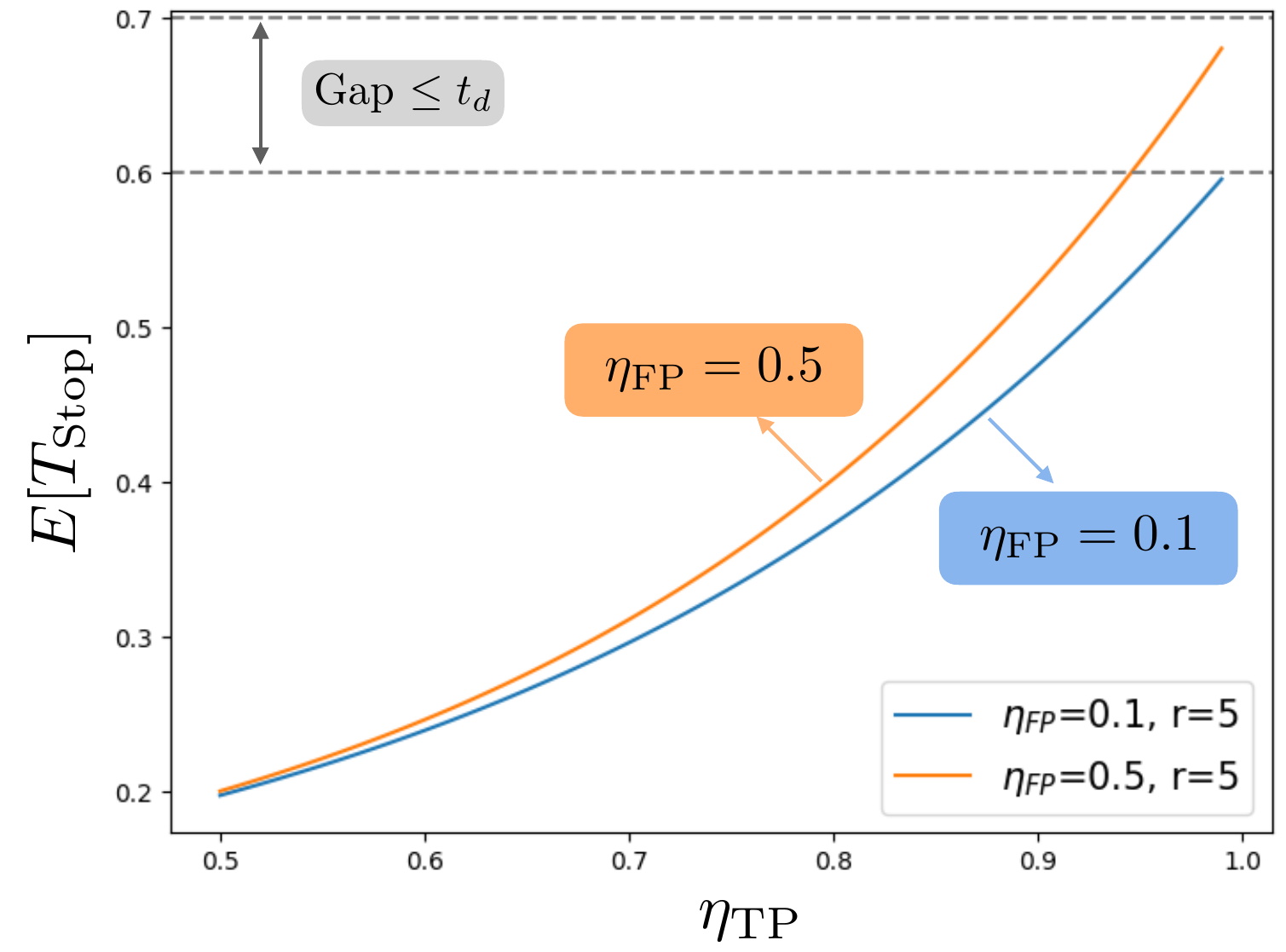}\vspace{-10pt}
    \caption{Illustration of expected stopping time of \textit{SPRINTER} as a function of true-positive rate ($\eta_{TP}$) for two different values of $\eta_{FP}$ when $r=5$ and $t_d = 0.1$. We can again observe that as long as $\eta_{FP}\leq 0.5$, the average stopping time (shown as the \textit{``Gap"}) generated by \textit{SPRINTER} never exceeds $t_d$.}
    \label{fig:latency_TP_FP}
\vspace{-10pt}
\end{figure}

\begin{theorem}\label{the:SPRINTER_expected_latency}
    The expected stopping time is given as:
    \begin{align}
        \label{eq:SPRINT_EXP_LATENCY}
        & \mathbb{E}[T_\text{Stop}] =  \frac{(1-\eta_{\text{TP}}^r) t_d}{1-\eta_\text{TP}} + \frac{\eta_\text{TP}^rt_d}{1-\eta_\text{FP}}.
    \end{align} 
\end{theorem}

Proof of the above theorem can be found in Section \ref{sec:lat_proof}. From Theorem \ref{the:SPRINTER_expected_latency}, we can observe that the expected stopping time increases as $\eta_{TP}$ increases. We also observe that the dependence on the false-positive rate $\eta_{FP}$ is marginal compared to $\eta_{TP}$. Furthermore, Fig. \ref{fig:latency_TP_FP} shows that as $\eta_{TP}$ approaches 1, even with a verifier having a relatively high $\eta_{FP}$ (e.g., 0.5), the additional expected stopping time remains minimal (Gap $\leq t_d$) when $r=5, t_d= 0.1$.

\textbf{Savings in Computation}. We now compare the total computational cost of \textit{one run of SD} versus \textit{one run of SPRINTER}. For simplicity, assume that one has a perfect verifier (i.e., $\eta_{FP}=0$ and $\eta_{TP}=1$). In SD, $\gamma$ tokens are first generated by the draft model, followed by parallel verification done by the target model. If we denote $F_d, F_t$ and $F_v$ as the number of flops to run the draft, target and verifier models once, then we have 
\begin{align}
    \text{SD-Flops}(\gamma) = \gamma F_d + \gamma F_t\\
    \text{\textit{SPRINTER}-Flops}(\gamma) = \gamma F_d + \gamma F_v +  F_t
\end{align}
If $F_{v}\ll F_d \ll F_t$, we can observe that computational savings from \textit{SPRINTER} can be significant and grow proportional to $(\gamma-1)F_t$ due to sequential verification through a low-complexity model (additional calculations showing the computational savings for model pairs in Section \ref{sec:flops_append}).

\subsection{Verifier Training and Architecture}
\label{sec:learn_veri}
\textbf{Verifier Training Methodology}. The verifier $V$ is a binary classifier trained to predict whether a token should be accepted or rejected. Specifically, $V()$ can take as input various latent features derived from the draft model (e.g. embedding of the draft token $x$, probability distribution of the LLM's vocabulary). We denote the input to the verifier as $s(x,\text{prefix})$. The data used for training the verifier can be prepared as follows: for a given prefix, a token $x$ is sampled from the draft model $M_q$. We also run the target model $M_p$ and compute $p(x)$.  Subsequently, binary labels are determined for each prefix and token pair by assigning $1$ if $\frac{q(x)}{p(x)} \leq 1$ and $0$ otherwise. \revision{ However, rather than comparing the ratio $\frac{q(x)}{p(x)}$ against a threshold of 1, we can increase (decrease) the threshold to $\lambda$ to bias the verifier to accept (reject) more draft tokens, which would increase (decrease) $\eta_\text{FP}$ or decrease (increase) $\eta_\text{TP}$.} Additionally, adjusting the inference threshold $\tau$ (see Algorithm \ref{alg:alg_short}) would achieve a similar effect. Thus, varying $\lambda$ and $\tau$ serve as the two hyperparameters which allow us to influence $(\eta_\text{FP},\eta_\text{TP})$. 

Our second observation is that during inference, \textit{SPRINTER} would face input as prefixes consisting of interleaved tokens generated in the past by draft and target models. Hence, during training, we expose the verifier to the following possible inputs: (a) an original prefix, (b) a prefix supplied with completions only from the draft model, (c) a prefix supplied with completions only from the target model (d) a prefix  with tokens from both draft and target models. To optimize the verifier's performance, we ensure an equal proportional of prefixes from each category.

\textbf{Verifier architecture used for evaluation}. For our experiments, the verifier was implemented as a fully connected linear layer followed by a sigmoid activation, containing significantly fewer parameters than $M_p$ and $M_q$. $V$ takes as input the last embedding of the previous token and is trained with an Adam optimizer assuming binary cross entropy loss. Our results indicate that a single layer is sufficient to achieve strong performance while maintaining high efficiency, as shown in Fig. \ref{fig:ROC_wiki_LM1B_12} on the Wiki-summary and LM1B datasets.  We observed that training thresholds of $\lambda=1.2$ enabled \textit{SPRINTER} to attain an effective performance on Wiki-Summary and LM1B datasets.

\vspace{-6pt}
\section{Experiments and Evaluation}
\label{sec:results}

\textbf{Evaluation goals.} To quantify the effectiveness of \textit{SPRINTER} we present the following set of results: (a) We measure the quality of responses generated by \textit{SPRINTER} using two performance metrics: the win-tie rates and ROUGE scores.  \revision{(b) We compare the latencies incurred by \textit{SPRINTER}, SD \cite{leviathan2023fast}, SpecDec++ \cite{huang2024specdec++}, AdaEDL \cite{agrawal2024adaedl}, and Mentored Decoding \cite{tranthien2024mentored} in generating text completions given a prefix. (c) We compare the effectiveness of \textit{SPRINTER} in completing tasks from Spec-Bench \cite{xia-etal-2024-unlocking} with  Medusa \cite{cai2024medusa} and Eagle2 \cite{li2406eagle}.} (d) To investigate the impact of verifier training and inference hyperparameters on \textit{SPRINTER}, we perform an ablation study to observe how $\eta_\text{TP}$ and $\eta_\text{FP}$ affect the ROC. (e)  As part of our qualitative comparison, Figure~\ref{fig: quality_latency_comput}(b) and Section~\ref{sec:prompts_ex} present example responses from \textit{SPRINTER} and SD, illustrating that \textit{SPRINTER} is capable of producing coherent and contextually appropriate tokens without strictly imitating the target model’s output distribution.

\textit{Code for \textit{SPRINTER} is available at \cite{codesprinter}}.


\textbf{Dataset and Model Architecture}. 
\revision{We use the WiKi-summary \cite{scheepers2018improving}, LM1B \cite{chelba2013one}, and Spec-Bench  \cite{xia-etal-2024-unlocking} datasets for evaluation.  Wiki-Summary is a collection of Wikipedia article summaries designed for text summarization tasks. 
The LM1B (One Billion Word Benchmark) Dataset is a large-scale corpus for language modeling and text generation tasks, extracted from news articles. Spec-Bench compiles questions from various LLM evaluation datasets (e.g. MT-Bench, GSM8K) covering different task categories including summarization and translation. We adopted a similar experimental setup as the prior works: \cite{github_specdec,chakraborty2024transferqstarprincipled}. We present results for three (draft, target) model pairs: GPT-Neo-125M \cite{GPT-Neo-125m}/GPT-Neo-1.3B \cite{GPT-Neo-1.3b}, GPT2-Small (124M param.) \cite{GPT2}/GPT2-XL (1.5B param.) \cite{GPT2-XL}, and Vicuna 68M \cite{Vicuna68m}/7B \cite{Vicuna7b}.}

\textbf{\underline{Quality Analysis}}. We investigate the quality of responses generated by \textit{SPRINTER} compared with standard SD. To quantify the completion quality of \textit{SPRINTER}, we report results using a) win-tie rates, and b) ROUGE scores. 

\textit{Win-tie rate} metric has been used extensively in LLM research \cite{rafailov2024direct,shen2024learning}. Win-tie rate measurements are taken by presenting GPT-4 with responses from two methods (\textit{SPRINTER}-generated vs SD-generated) and prompting it to decide which response is better based on criteria provided by the user (additional details on win-tie rates are provided in Section \ref{sec:add_quali_results}).   Table \ref{tab:winrate} reports win-tie rate comparisons of \textit{SPRINTER} with SD,  where  GPT-4 is provided with an initial prefix and completions from \textit{SPRINTER} and SD. \revision{Approximately 30\% of each prompt was given as input to each technique, which were constrained to generate 20 tokens per prompt.} The table indicates that \textit{SPRINTER} using the GPT-Neo model pair can generate responses of comparable quality to SD. This is especially observed on Wiki-Summary, which wins on-average 45.2\% of the time against SD, indicating that \textit{SPRINTER} suffers minimal quality degradation. 

\begin{table}[t]
    \centering
    \resizebox{0.45\textwidth}{!}{%
    \begin{tabular}{l l c c}
        \hline
        $M_q/M_p$ & \textcolor{teal}{WiKi-Summary} & \textcolor{blue}{LM1B} \\
        \hline
        \textcolor{purple}{GPT-Neo-125M/1.3B} &
        \cellcolor{teal!15}45.2 $\pm$ 3.12 &
        \cellcolor{blue!15}35.4 $\pm$ 5.08 \\
        
        \textcolor{orange}{GPT2-Small/XL} &
        \cellcolor{teal!15}41.0 $\pm$ 6.82 &
        \cellcolor{blue!15}41.6 $\pm$ 3.98 \\
        \hline
    \end{tabular}}
    \caption{Average win-tie rates of \textit{SPRINTER} against SD for GPT-Neo-125M/GPT-Neo-1.3B and GPT2-Small/GPT2-XL pairs.}
    \label{tab:winrate}
\end{table}

\begin{figure}[t]
    \centering
    \includegraphics[width=0.9\linewidth]{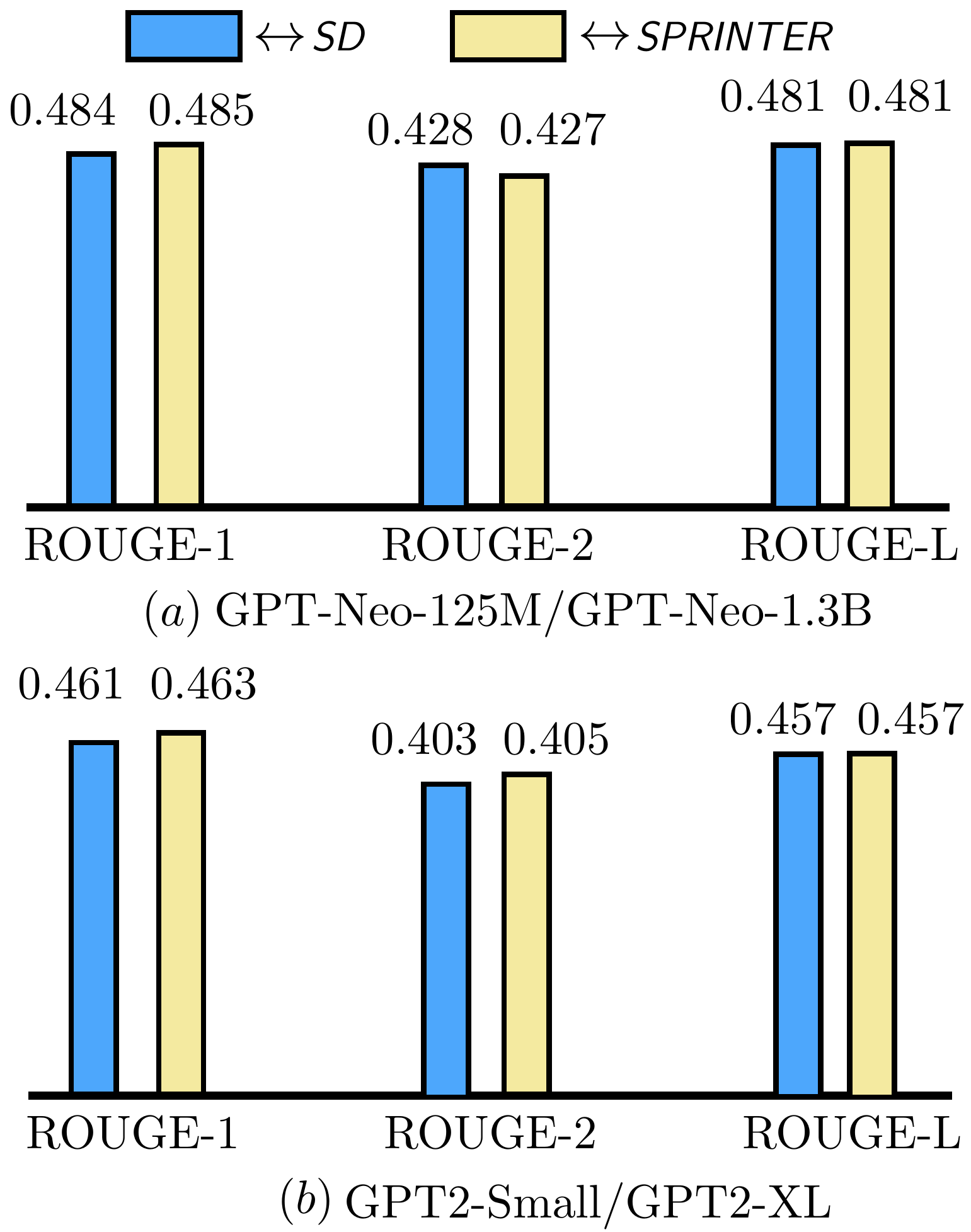}
    \caption{The ROUGE metrics (ROUGE-1, ROUGE-2, ROUGE-2) of \textit{SPRINTER} vs SD for (a) GPT-Neo-125M/GPT-Neo-1.3B and (b) GPT2-Small/GPT2-XL model pairs. This demonstrates that even with faster inference speeds, \textit{SPRINTER} experiences only a minimal drop in quality compared to SD.}
    \label{fig:Rouge_overall}
    \vspace{-5pt}
\end{figure}
\arrayrulecolor{black}
\setlength{\arrayrulewidth}{0.4pt}
\begin{table*}[t]
    \centering
    \begin{tabular}{|l|l|c|c|c|c|c|c|}
      \hline
        Model Pair & &\multicolumn{3}{|c|}{\textcolor{teal}{WiKi-Summary}} & \multicolumn{3}{|c|}{\textcolor{blue}{LM1B}}\\
        \hline
       \textcolor{purple}{\multirow{5}{5em}{ GPT-Neo-125M / GPT-Neo-
1.3B}}&Methods  & Avg Tokens & Time (ms) & Speedup  & Avg Tokens & Time (ms) & Speedup \\
       \hhline{~-------}
       &$SD$  & \cellcolor{teal!15}  2.17 & \cellcolor{teal!15} 0.92  $\pm$  0.32& \cellcolor{teal!15} 1x & \cellcolor{blue!15}1.95 & \cellcolor{blue!15}0.84  $\pm$   0.21 & \cellcolor{blue!15}1x\\
       \hhline{~-------}
       &$Specdec^{++}$ & \cellcolor{teal!15} 1.16 & \cellcolor{teal!15} 0.75 $\pm$  0.25 &  \cellcolor{teal!15} 1.23x & \cellcolor{blue!15}1.08 & \cellcolor{blue!15}0.61  $\pm$  0.24 & \cellcolor{blue!15}1.38x\\
       \hhline{~-------}
       &$AdaEDL$ &\cellcolor{teal!15} 2.66 & \cellcolor{teal!15} 0.72  $\pm$  0.30 &\cellcolor{teal!15} 1.28x & \cellcolor{blue!15}1.91 & \cellcolor{blue!15}0.56   $\pm$  0.12 & \cellcolor{blue!15}1.50x \\
       \hhline{~-------}
       &$MentoredDec$ &\cellcolor{teal!15} 3.39 & \cellcolor{teal!15}0.76 $\pm$ 0.18 & \cellcolor{teal!15} 1.21x & \cellcolor{blue!15}3.11 & \cellcolor{blue!15}0.62 $\pm$ 0.13 & \cellcolor{blue!15}1.35x \\
    \hhline{~-------}
       &\textit{SPRINTER} & \cellcolor{teal!15}\textbf{11.10} & \cellcolor{teal!15}\textbf{0.56  $\pm$  0.22} & \cellcolor{teal!15}\textbf{1.64x} & \cellcolor{blue!15}\textbf{8.32} & \cellcolor{blue!15}\textbf{0.46  $\pm$  0.09} & \cellcolor{blue!15}\textbf{1.83x}  \\
       \hline
       \hline
       \textcolor{orange}{\multirow{5}{5em}{GPT2-Small/ GPT2-XL}}&$SD$  &  \cellcolor{teal!15}2.01 & \cellcolor{teal!15}0.84 $\pm$ 0.24 & \cellcolor{teal!15}1x & \cellcolor{blue!15}2.05 & \cellcolor{blue!15}0.73 $\pm$ 0.15 &\cellcolor{blue!15} 1x\\
       \hhline{~-------}
       &$Specdec^{++}$ &\cellcolor{teal!15}1.82  &\cellcolor{teal!15} 0.72 $\pm$ 0.23 & \cellcolor{teal!15}1.18x & \cellcolor{blue!15}1.04 & \cellcolor{blue!15}0.61 $\pm$ 0.28  & \cellcolor{blue!15}1.20x\\
       \hhline{~-------}
       &$AdaEDL$ & \cellcolor{teal!15}1.96 & \cellcolor{teal!15}0.73 $\pm$ 0.25 &\cellcolor{teal!15} 1.16x &\cellcolor{blue!15} 1.52 & \cellcolor{blue!15}0.64 $\pm$ 0.19 & \cellcolor{blue!15}1.14x \\
       \hhline{~-------}
       &$MentoredDec$ &\cellcolor{teal!15} 3.20 &\cellcolor{teal!15} 0.60 $\pm$ 0.19 &\cellcolor{teal!15} 1.40x &\cellcolor{blue!15} 3.12 & \cellcolor{blue!15}0.55 $\pm$ 0.07 & \cellcolor{blue!15}1.33x \\
       \hhline{~-------}
       &\textit{SPRINTER} & \cellcolor{teal!15}\textbf{10.83} & \cellcolor{teal!15}\textbf{0.49 $\pm$ 0.20} & \cellcolor{teal!15}\textbf{1.69x} & \cellcolor{blue!15}\textbf{6.56} & \cellcolor{blue!15}\textbf{0.44 $\pm$ 0.12} & \cellcolor{blue!15}\textbf{1.66x}  \\
       \hline
    \end{tabular}
    \caption{\revision{Latency speedups for \textit{SPRINTER} relative to SD \cite{leviathan2023fast}, AdaEDL \cite{agrawal2024adaedl}, Specdec++ \cite{huang2024specdec++} and MentoredDec \cite{tranthien2024mentored} using GPT families as the draft/target models on Wiki-Summary and LM1B datasets.} }
    \label{tab:latency_wiki}
\end{table*}

\textit{ROUGE score comparison}. To further illustrate the comparable quality between \textit{SPRINTER} and SD, Figures \ref{fig:Rouge_overall} show ROUGE scores \cite{lin-2004-rouge} of the responses made by \textit{SPRINTER} and SD respectively with the reference summaries in Wiki-Summary for both GPT-Neo and GPT2 model pairs. The ROUGE scores measure differing levels of similarity  between a provided ``candidate" summary  and reference summary \cite{lin-2004-rouge}. 
For both model pairs, the figures indicate that \textit{SPRINTER} is able to generate responses that attain very similar ROUGE scores to SD.

\revision{
\noindent\textbf{\underline{Latency Speedups}}. Table \ref{tab:latency_wiki} reports results comparing the latency speedups achieved with \textit{SPRINTER} relative to the methods shown in Figure \ref{fig:frameworkSPrinter}(c) for the Wiki-Summary and LM1B datasets. 
Given a prefix, 20 additional tokens were generated per prompt by each method. We report the ``Avg Token" as the accepted token generated by $M_q$ per single run of the sampling process. As the table indicates, \textit{SPRINTER} achieves higher speedup improvements relative to SD compared to the other baselines, without being restricted to generate tokens that match the target model distribution.}

\begin{table}[t]
\centering
\resizebox{0.45\textwidth}{!}{
\begin{tabular}{|l|c|c|c|}
\hline
\textcolor{teal}{\textit{SPRINTER} vs} & Win-Tie rate & Speedup & Flops \\
\hline
SD        & \cellcolor{gray!15}41  & \cellcolor{gray!15}1.64x & \cellcolor{gray!15}$F_d(\gamma) +  F_t(\gamma)$ \\
Medusa    & \cellcolor{yellow!20}38  &\cellcolor{yellow!20} 1.51x & \cellcolor{yellow!20}$K F_d(\gamma) + F_t(\tau_n )$\\
Eagle2    & \cellcolor{orange!20}32  &\cellcolor{orange!20} 1.17x &\cellcolor{orange!20} $F_t(\tau_n)  + \tau_d F_d(\frac{\tau_n}{\tau_d})$\\
Mentored Dec & \cellcolor{brown!10}52 & \cellcolor{brown!10}1.35x & \cellcolor{brown!10}$F_d(\gamma) +  F_t(\gamma)$\\
\hline
\end{tabular}}
\caption{\revision{Comparison of \textit{SPRINTER} with state-of-the-art methods (Medusa \cite{cai2024medusa}, Eagle2 \cite{li2406eagle} and Mentored Decoding \cite{tranthien2024mentored}) in terms of win-tie rate and speedup. Medusa and Eagle2 FLOPS to process $\gamma$ tokens from \cite{christopher2024speculative}, where $\tau_n, \tau_d, K$ are the number of nodes/depth of the draft tree and Medusa heads respectively.}}
\label{tab:sprinter_comparison}
\end{table}

\revision{
\noindent\textbf{\underline{Quality-Latency-Flops Tradeoffs}}.
We further examine the trade-offs between quality and latency using Spec-Bench \cite{xia-etal-2024-unlocking} with the Vicuna 68M/7B pair, comparing \textit{SPRINTER} against SD, Medusa \cite{cai2024medusa}, and Eagle2 \cite{li2406eagle}. For \textit{SPRINTER}, we employ a verifier trained on LM1B responses from the GPT2-Small/GPT2-XL pair. Each method was constrained to generate 20 tokens per task. We also compare \textit{SPRINTER} and the lossy method Mentored Decoding \cite{tranthien2024mentored} on Wiki-Summaries using the GPT-Neo model pair. As shown in Table~\ref{tab:sprinter_comparison}, \textit{SPRINTER} requires significantly fewer FLOPs than SD and other variants, and achieves lower latency while preserving high output quality.



\textbf{\underline{Verifier Ablation Study}}. 
Rather than strictly forcing the verifier to only accept draft tokens if the underlying ratio $\left(\frac{q_i(x|\text{prefix})}{p_i(x|\text{prefix})}\right) \leq 1$, the threshold of $1$ could be changed to a parameter $\lambda>1$; allowing for more tokens to be deemed acceptable, providing another method for \textit{SPRINTER} to generate tokens that deviate from the target distribution and accelerate inference. 
We show ROC curves for the verifier on the LM1B and Wiki dataset when trained at thresholds $\lambda = $ 1, 1.2, and 1.5 respectively in Section \ref{sec:add_results}.

\begin{table}[t]
\centering
\resizebox{0.47\textwidth}{!}{
\begin{tabular}{|c|cc|cc|}
\hline
\multirow{2}{5em}{ Eval\textbackslash Train }& \multicolumn{2}{c|}{\textbf{Latency}} & \multicolumn{2}{c|}{\textbf{Quality}} \\

     & \textbf{WiKi} & \textbf{Lm1b} & \textbf{WiKi} & \textbf{Lm1b} \\
\hline
\textbf{WiKi}  & \cellcolor{violet!15} 0.56 ± 0.22 & 0.57 ± 0.26 & \cellcolor{violet!15} 45.2 ± 3.12 & 47.6 ± 5.78 \\
\textbf{Lm1b}  & 0.51 ± 0.09 & \cellcolor{violet!15} 0.46 ± 0.09 & 38.2 ± 1.94 & \cellcolor{violet!15}39.6 ± 1.94 \\
\hline
\end{tabular}}
\caption{\revision{Comparison of Latency and Quality between WiKi and Lm1b datasets. The verifier maintains comparable latency and quality when trained on one dataset and evaluated on another. Highlighted entries show the average win-tie rates using verifiers trained/evaluated on the same dataset.}}
\label{tab:latency_quality}
\end{table}
\textbf{\underline{Verifier Transferability}.} We also explore the generalization capability of the verifier and observe that it can be effectively transferred across tasks and datasets. As shown in Table \ref{tab:latency_quality}, the verifier achieves comparable latency and quality (measured by win-tie rate) when a model trained on one dataset is used for inference on another. For example, evaluating a verifier trained on Lm1b attains an average win rate of $47.6$ and $39.6$ on Wiki and Lm1b respectively.}

\vspace{-8pt}
\section{Conclusion}
\vspace{-5pt}
We introduced \textit{SPRINTER}, a sampling framework designed to accelerate LLM inference by leveraging a draft-target model pair along with a lightweight verifier. Our theoretical analysis highlights the trade-offs between inference speed, computational efficiency, and output quality, demonstrating how verifier characteristics, such as false-positive and true-positive rates, influence performance. Through extensive experiments on multiple datasets and model pairs, we showed that \textit{SPRINTER} significantly reduces latency while maintaining high-quality outputs. Our result show that sequential approximate verification can be effective in balancing efficiency and quality, making it a promising approach for scalable and efficient LLM deployment.

\vspace{-6.5pt}
\section*{Acknowledgments}
\vspace{-6.5pt}
This work was supported by NSF grants CAREER 1651492, CCF-2100013, CNS-2209951, CNS-1822071, CNS-2317192, and by the U.S. Department of Energy, Office of Science, Office of Advanced Scientific Computing under Award Number DE-SC-ERKJ422, and NIH Award R01-CA261457-01A1.
\vspace{-6.5pt}
\section*{Impact Statement}
\vspace{-6.5pt}
With the growing use of LLMs in different applications, and the scaling in their size and complexity, inference speed remains a critical bottleneck in real-world applications, affecting latency-sensitive tasks such as conversational AI, real-time translation, and autonomous systems. By introducing the idea of low-complexity sequential verification within the context of speculative decoding, \textit{SPRINTER} can reduce response times and lowers computational costs without significantly compromising on the quality. \textit{SPRINTER} ensures that the target model is not called frequently, resulting in a reduction in energy consumption compared to standard speculative decoding, while ensuring minimal degradation in the quality of its generated responses. Additionally, \textit{SPRINTER} is able to generate these responses at significantly higher speeds compared to SD, making it more feasible for LLMs to be used in time-sensitive applications. 
\vspace{-10pt}

\nocite{langley00}

\bibliography{example_paper}

\begin{thebibliography}{37}
\providecommand{\natexlab}[1]{#1}
\providecommand{\url}[1]{\texttt{#1}}
\expandafter\ifx\csname urlstyle\endcsname\relax
  \providecommand{\doi}[1]{doi: #1}\else
  \providecommand{\doi}{doi: \begingroup \urlstyle{rm}\Url}\fi

\bibitem[Agrawal et~al.(2024)Agrawal, Jeon, and Lee]{agrawal2024adaedl}
Agrawal, S., Jeon, W., and Lee, M.
\newblock Adaedl: Early draft stopping for speculative decoding of large language models via an entropy-based lower bound on token acceptance probability.
\newblock In \emph{NeurIPS Efficient Natural Language and Speech Processing Workshop}, pp.\  355--369. PMLR, 2024.

\bibitem[Bachmann et~al.(2025)Bachmann, Anagnostidis, Pumarola, Georgopoulos, Sanakoyeu, Du, Sch{\"o}nfeld, Thabet, and Kohler]{judge}
Bachmann, G., Anagnostidis, S., Pumarola, A., Georgopoulos, M., Sanakoyeu, A., Du, Y., Sch{\"o}nfeld, E., Thabet, A., and Kohler, J.
\newblock Judge decoding: Faster speculative sampling requires going beyond model alignment.
\newblock \emph{arXiv preprint arXiv:2501.19309}, 2025.

\bibitem[Cai et~al.(2024)Cai, Li, Geng, Peng, Lee, Chen, and Dao]{cai2024medusa}
Cai, T., Li, Y., Geng, Z., Peng, H., Lee, J.~D., Chen, D., and Dao, T.
\newblock Medusa: Simple llm inference acceleration framework with multiple decoding heads.
\newblock \emph{arXiv preprint arXiv:2401.10774}, 2024.

\bibitem[Casson(2023)]{casson2023transformerflops}
Casson, A.
\newblock Transformer flops.
\newblock 2023.
\newblock URL \url{https://adamcasson.com/posts/transformer-flops}.

\bibitem[Chakraborty et~al.(2024)Chakraborty, Ghosal, Yin, Manocha, Wang, Bedi, and Huang]{chakraborty2024transferqstarprincipled}
Chakraborty, S., Ghosal, S.~S., Yin, M., Manocha, D., Wang, M., Bedi, A.~S., and Huang, F.
\newblock Transfer q star: Principled decoding for llm alignment, 2024.
\newblock URL \url{https://arxiv.org/abs/2405.20495}.

\bibitem[Chelba et~al.(2013)Chelba, Mikolov, Schuster, Ge, Brants, Koehn, and Robinson]{chelba2013one}
Chelba, C., Mikolov, T., Schuster, M., Ge, Q., Brants, T., Koehn, P., and Robinson, T.
\newblock One billion word benchmark for measuring progress in statistical language modeling.
\newblock \emph{arXiv preprint arXiv:1312.3005}, 2013.

\bibitem[Christopher et~al.(2024)Christopher, Bartoldson, Ben-Nun, Cardei, Kailkhura, and Fioretto]{christopher2024speculative}
Christopher, J.~K., Bartoldson, B.~R., Ben-Nun, T., Cardei, M., Kailkhura, B., and Fioretto, F.
\newblock Speculative diffusion decoding: Accelerating language generation through diffusion.
\newblock \emph{arXiv preprint arXiv:2408.05636}, 2024.

\bibitem[EleutherAI(2024{\natexlab{a}})]{GPT-Neo-1.3b}
EleutherAI.
\newblock Eleutherai/gpt-neo-1.3b, 2024{\natexlab{a}}.
\newblock URL \url{https://huggingface.co/EleutherAI/gpt-neo-1.3B}.
\newblock Accessed: 011-2024.

\bibitem[EleutherAI(2024{\natexlab{b}})]{GPT-Neo-125m}
EleutherAI.
\newblock Eleutherai/gpt-neo-125m, 2024{\natexlab{b}}.
\newblock URL \url{https://huggingface.co/EleutherAI/gpt-neo-125m}.
\newblock Accessed: 011-2024.

\bibitem[Fang(2024)]{github_specdec}
Fang, J.
\newblock Llmspeculativesampling.
\newblock \url{https://github.com/feifeibear/LLMSpeculativeSampling}, 2024.
\newblock Accessed: 012-2024.

\bibitem[Hoffmann et~al.(2022)Hoffmann, Borgeaud, Mensch, Buchatskaya, Cai, Rutherford, de~Las~Casas, Hendricks, Welbl, Clark, Hennigan, Noland, Millican, van~den Driessche, Damoc, Guy, Osindero, Simonyan, Elsen, Rae, Vinyals, and Sifre]{hoffmann2022trainingcomputeoptimallargelanguage}
Hoffmann, J., Borgeaud, S., Mensch, A., Buchatskaya, E., Cai, T., Rutherford, E., de~Las~Casas, D., Hendricks, L.~A., Welbl, J., Clark, A., Hennigan, T., Noland, E., Millican, K., van~den Driessche, G., Damoc, B., Guy, A., Osindero, S., Simonyan, K., Elsen, E., Rae, J.~W., Vinyals, O., and Sifre, L.
\newblock Training compute-optimal large language models, 2022.
\newblock URL \url{https://arxiv.org/abs/2203.15556}.

\bibitem[Huang et~al.(2024)Huang, Guo, and Wang]{huang2024specdec++}
Huang, K., Guo, X., and Wang, M.
\newblock Specdec++: Boosting speculative decoding via adaptive candidate lengths.
\newblock \emph{arXiv preprint arXiv:2405.19715}, 2024.

\bibitem[Jang et~al.(2024)Jang, Park, Yang, Jung, Yun, Kundu, Kim, and Yang]{jang2024lanternacceleratingvisualautoregressive}
Jang, D., Park, S., Yang, J.~Y., Jung, Y., Yun, J., Kundu, S., Kim, S.-Y., and Yang, E.
\newblock Lantern: Accelerating visual autoregressive models with relaxed speculative decoding, 2024.
\newblock URL \url{https://arxiv.org/abs/2410.03355}.

\bibitem[Kim et~al.(2024)Kim, Mangalam, Moon, Malik, Mahoney, Gholami, and Keutzer]{kim2024speculative}
Kim, S., Mangalam, K., Moon, S., Malik, J., Mahoney, M.~W., Gholami, A., and Keutzer, K.
\newblock Speculative decoding with big little decoder.
\newblock \emph{Advances in Neural Information Processing Systems}, 36, 2024.

\bibitem[Laboratory(2024)]{Vicuna68m}
Laboratory, N.~K.
\newblock double7/vicuna-68m, 2024.
\newblock URL \url{https://huggingface.co/double7/vicuna-68m}.
\newblock Accessed: 04-2025.

\bibitem[Langley(2000)]{langley00}
Langley, P.
\newblock Crafting papers on machine learning.
\newblock In Langley, P. (ed.), \emph{Proceedings of the 17th International Conference on Machine Learning (ICML 2000)}, pp.\  1207--1216, Stanford, CA, 2000. Morgan Kaufmann.

\bibitem[Leviathan et~al.(2023)Leviathan, Kalman, and Matias]{leviathan2023fast}
Leviathan, Y., Kalman, M., and Matias, Y.
\newblock Fast inference from transformers via speculative decoding.
\newblock In \emph{International Conference on Machine Learning}, pp.\  19274--19286. PMLR, 2023.

\bibitem[Li et~al.()Li, Wei, Zhang, and Zhang]{li2406eagle}
Li, Y., Wei, F., Zhang, C., and Zhang, H.
\newblock Eagle-2: Faster inference of language models with dynamic draft trees, 2024b.
\newblock \emph{URL https://arxiv. org/abs/2406.16858}.

\bibitem[Li et~al.(2024)Li, Wei, Zhang, and Zhang]{li2024eagle}
Li, Y., Wei, F., Zhang, C., and Zhang, H.
\newblock Eagle: Speculative sampling requires rethinking feature uncertainty.
\newblock \emph{arXiv preprint arXiv:2401.15077}, 2024.

\bibitem[Lin(2004)]{lin-2004-rouge}
Lin, C.-Y.
\newblock {ROUGE}: A package for automatic evaluation of summaries.
\newblock In \emph{Text Summarization Branches Out}, pp.\  74--81, Barcelona, Spain, July 2004. Association for Computational Linguistics.
\newblock URL \url{https://aclanthology.org/W04-1013/}.

\bibitem[LMSYS(2023)]{Vicuna7b}
LMSYS.
\newblock lmsys/vicuna-7b-v1.3, 2023.
\newblock URL \url{https://huggingface.co/lmsys/vicuna-7b-v1.3}.
\newblock Accessed: 04-2025.

\bibitem[Lu et~al.(2024)Lu, Zeng, Ma, Yu, and Levorato]{lu2024improving}
Lu, X., Zeng, Y., Ma, F., Yu, Z., and Levorato, M.
\newblock Improving multi-candidate speculative decoding.
\newblock \emph{arXiv preprint arXiv:2409.10644}, 2024.

\bibitem[Mamou et~al.(2024)Mamou, Pereg, Korat, Berchansky, Timor, Wasserblat, and Schwartz]{mamou2024dynamicspeculationlookaheadaccelerates}
Mamou, J., Pereg, O., Korat, D., Berchansky, M., Timor, N., Wasserblat, M., and Schwartz, R.
\newblock Dynamic speculation lookahead accelerates speculative decoding of large language models, 2024.
\newblock URL \url{https://arxiv.org/abs/2405.04304}.

\bibitem[Melcer et~al.(2024)Melcer, Gonugondla, Perera, Qian, Chiang, Wang, Jain, Garg, Ma, and Deoras]{melcer2024approximately}
Melcer, D., Gonugondla, S., Perera, P., Qian, H., Chiang, W.-H., Wang, Y., Jain, N., Garg, P., Ma, X., and Deoras, A.
\newblock Approximately aligned decoding.
\newblock \emph{arXiv preprint arXiv:2410.01103}, 2024.

\bibitem[Niu et~al.(2024)Niu, Chen, Li, Feng, Bi, Liu, and Peng]{niu2024text}
Niu, Q., Chen, K., Li, M., Feng, P., Bi, Z., Liu, J., and Peng, B.
\newblock From text to multimodality: Exploring the evolution and impact of large language models in medical practice.
\newblock \emph{arXiv preprint arXiv:2410.01812}, 2024.

\bibitem[OpenAI(2024{\natexlab{a}})]{GPT2}
OpenAI.
\newblock openai-community/gpt2, 2024{\natexlab{a}}.
\newblock URL \url{https://huggingface.co/openai-community/gpt2}.
\newblock Accessed: 09-2024.

\bibitem[OpenAI(2024{\natexlab{b}})]{GPT2-XL}
OpenAI.
\newblock openai-community/gpt2-xl, 2024{\natexlab{b}}.
\newblock URL \url{https://huggingface.co/openai-community/gpt2-xl}.
\newblock Accessed: 09-2024.

\bibitem[Rafailov et~al.(2024)Rafailov, Sharma, Mitchell, Manning, Ermon, and Finn]{rafailov2024direct}
Rafailov, R., Sharma, A., Mitchell, E., Manning, C.~D., Ermon, S., and Finn, C.
\newblock Direct preference optimization: Your language model is secretly a reward model.
\newblock \emph{Advances in Neural Information Processing Systems}, 36, 2024.

\bibitem[Scheepers et~al.(2018)Scheepers, Kanoulas, and Gavves]{scheepers2018improving}
Scheepers, T., Kanoulas, E., and Gavves, E.
\newblock Improving word embedding compositionality using lexicographic definitions.
\newblock In \emph{Proceedings of the 2018 World Wide Web Conference}, pp.\  1083--1093, 2018.

\bibitem[Shen et~al.(2024)Shen, Lang, Wang, Kim, and Sontag]{shen2024learning}
Shen, S.~Z., Lang, H., Wang, B., Kim, Y., and Sontag, D.
\newblock Learning to decode collaboratively with multiple language models.
\newblock \emph{arXiv preprint arXiv:2403.03870}, 2024.

\bibitem[Sprinter(2025)]{codesprinter}
Sprinter.
\newblock Code for {S}{P}{R}{I}{N}{T}{E}{R}, 2025.
\newblock URL \url{https://github.com/MeiyuZhong/SPRINTER.git}.

\bibitem[Sun et~al.(2024)Sun, Suresh, Ro, Beirami, Jain, and Yu]{sun2024spectr}
Sun, Z., Suresh, A.~T., Ro, J.~H., Beirami, A., Jain, H., and Yu, F.
\newblock Spectr: Fast speculative decoding via optimal transport.
\newblock \emph{Advances in Neural Information Processing Systems}, 36, 2024.

\bibitem[Tang et~al.(2024)Tang, Bi, Xu, Song, Liang, Wang, Zhang, An, Lin, Zhu, Vosoughi, Huang, Zhang, Liu, Feng, Zheng, Zhang, Luo, Luo, and Xu]{tang2024videounderstandinglargelanguage}
Tang, Y., Bi, J., Xu, S., Song, L., Liang, S., Wang, T., Zhang, D., An, J., Lin, J., Zhu, R., Vosoughi, A., Huang, C., Zhang, Z., Liu, P., Feng, M., Zheng, F., Zhang, J., Luo, P., Luo, J., and Xu, C.
\newblock Video understanding with large language models: A survey, 2024.
\newblock URL \url{https://arxiv.org/abs/2312.17432}.

\bibitem[Tran-Thien(2024)]{tranthien2024mentored}
Tran-Thien, V.
\newblock An optimal lossy variant of speculative decoding, June 2024.
\newblock URL \url{https://huggingface.co/blog/vivien/optimal-lossy-variant-of-speculative-decoding}.
\newblock Accessed: 2025-05-09.

\bibitem[Xia et~al.(2024)Xia, Yang, Dong, Wang, Li, Ge, Liu, Li, and Sui]{xia-etal-2024-unlocking}
Xia, H., Yang, Z., Dong, Q., Wang, P., Li, Y., Ge, T., Liu, T., Li, W., and Sui, Z.
\newblock Unlocking efficiency in large language model inference: A comprehensive survey of speculative decoding.
\newblock In Ku, L.-W., Martins, A., and Srikumar, V. (eds.), \emph{Findings of the Association for Computational Linguistics ACL 2024}, pp.\  7655--7671, Bangkok, Thailand and virtual meeting, August 2024. Association for Computational Linguistics.
\newblock \doi{10.18653/v1/2024.findings-acl.456}.
\newblock URL \url{https://aclanthology.org/2024.findings-acl.456}.

\bibitem[Yin et~al.(2024)Yin, Chen, Huang, and Wang]{yin2024theoretical}
Yin, M., Chen, M., Huang, K., and Wang, M.
\newblock A theoretical perspective for speculative decoding algorithm.
\newblock \emph{arXiv preprint arXiv:2411.00841}, 2024.

\bibitem[Zingale \& Kalita(2024)Zingale and Kalita]{zingale2024language}
Zingale, J. and Kalita, J.
\newblock Language model sentence completion with a parser-driven rhetorical control method.
\newblock \emph{arXiv preprint arXiv:2402.06125}, 2024.

\end{thebibliography}
\bibliographystyle{icml2025}

\newpage
\appendix
\onecolumn
\section{Appendix}
The Appendix is organized as follows: 

\ref{sec:31proof} Proof of Theorem \ref{the:SPRINTER} (Analysis of the distribution of the tokens generated by
SPRINTER)

\ref{sec:expec_proof} Proof of Theorem \ref{the:SPRINTER_expected_tokens} (Analysis of the expected number of generated tokens from SPRINTER)

\ref{sec:lat_proof} Proof of Theorem \ref{the:SPRINTER_expected_latency} (SPRINTER latency analysis)

\ref{sec:Acc_rate} Additional theoretical results

\ref{sec:algorithm_hyper} Full SPRINTER algorithm and hyperparameter tuning

\ref{sec:relatedworks_append} More related works

\ref{sec:prompts_ex} Additional Examples of Responses generated by Sprinter and SD

\ref{sec:flops_append} Flops explainations and evaluations

\ref{sec:add_quali_results} Prompt Design for Win-tie Rate Evaluation

\ref{sec:add_results} Additional experimental results

\subsection{Proof of Theorem \ref{the:SPRINTER} }
\label{sec:31proof}
\begingroup
\renewcommand{\thetheorem}{\ref{the:SPRINTER}} 
\begin{theorem}
The probability of a token $x$ being chosen when running SPRINTER is given as
    \begin{align}
    \label{eq:SPRINT_DIST}
        p_\text{SPRINTER}(x) = (1-\eta_\text{FP})p(x) + \eta_\text{FP} q(x),
    \end{align}
    where $\eta_\text{FP}$ is the false positive rate of the verifier. Furthermore, the total-variation distance between the target and SPRINTER distributions is $d_\text{TV}(p,p_\text{SPRINTER}) = \eta_\text{FP}d_\text{TV}(p,q)$.
\end{theorem}
\endgroup


\begin{proof}
    We denote $\eta_\text{FP}$ and $\eta_\text{FN}$ as the false positive and false negative rate of the verifier respectively. We define $x_\text{acc} = \{ x | \frac{q(x)}{p(x)}\leq 1 \}$ and $x_\text{rej} = \{ x | \frac{q(x)}{p(x)}>1 \}$ to represent the sets of tokens generated by the draft model that should be accepted and rejected by the verifier respectively.

    In general, we consider two cases: (a) the token should be rejected while it is accepted by verifier $V$, (b) the token should be accepted but it is rejected by $V$. Let us first consider the case when the token should be rejected (i.e. $x \in x_\text{rej}$). If $V(s(x,\text{prefix})) = 0$, meaning that the verifier predicts that ${\frac{q(x)}{p(x)}} > 1$, then we reject the token $x$ sampled from $q(x)$ with probability $1-\frac{p(x)}{q(x)}$ and re-sample $x$ from an adjusted distribution $\text{norm}(\max(0, p(x) - q(x)))$. Therefore, the probability that token $x$ is accepted is:
    \begin{align}
        \frac{p(x)}{q(x)} \times (1-\eta_\text{FP}) \times q(x). \label{eq: firstcase_z1}
    \end{align} If $V(s(x,\text{prefix})) =1$, meaning that the verifier predicts that  ${\frac{q(x)}{p(x)}} \leq 1$, we accept the token, which occurs with probability:
    \begin{align}
        \eta_\text{FP} \times q(x).  \label{eq: firstcase_z0}
    \end{align}
    Therefore,  combining \eqref{eq: firstcase_z1} and \eqref{eq: firstcase_z0}, the probability of a token $x \in x_\text{rej}$ being accepted under SPRINTER is:
    \begin{align}
      (1-\eta_\text{FP}) p(x) + \eta_\text{FP} q(x). \label{eq: theorem1_side1}
    \end{align}

    Now we consider the case when the token should be accepted (i.e. $x\in x_\text{acc}$) . If $V(s(x,\text{prefix})) = 0$, with probability $1-\eta_\text{TP}$, we reject the token $x$ and call the larger model $M_p$, which will verify that indeed $\frac{q(x)}{p(x)} \leq 1$. Therefore, though the verifier made a mistake, it can be corrected by $M_p$. The probability that token $x$ is accepted is:
    \begin{align}
        q(x) \times (1-\eta_\text{TP}). \label{eq: seccase_z1}
    \end{align}
    If $V(s(x,\text{prefix})) = 1$, we accept token $x$ with probability:
    \begin{align}
        q(x) \times \eta_\text{TP}. \label{eq: sec_acc_case_z1}
    \end{align}

    While the above scenarios relied directly on the decision of the verifier given an acceptable token, there is an additional scenario where a token $x \in x_\text{acc}$ is accepted despite not being initially sampled from $M_q$. Assume that a token sampled from the draft model should be rejected (i.e. $x \in x_\text{rej}$) and the verifier accurately predicts to invoke the target model. Under this event, it would be possible for a token that is acceptable to be re-sampled from the adjusted distribution.    
    The probability of token $x \in x_\text{acc}$ being accepted under this scenario is:
    \begin{align}
         &(1-\eta_\text{FP}) \times \left( \sum_{x_{\text{rej}}} q(x_\text{rej}) (1-\frac{p(x_\text{rej})}{q(x_\text{rej})}) \cdot \frac{p(x)-q(x)}{\sum_{x_\text{acc}} p(x_\text{acc}) - q(x_\text{acc})} \right)  \nonumber\\
         & = (1-\eta_\text{FP}) \times (p(x)-q(x))\times  \frac{\sum_{x_\text{rej}} p(x_\text{rej}) -q(x_\text{rej})}{\sum_{x_\text{acc}} p(x_\text{acc}) - q(x_\text{acc})}\nonumber\\
         & = (1-\eta_\text{FP}) \times (p(x)-q(x)). \label{eq: seccase_z0}
    \end{align}
    Overall, combining \eqref{eq: seccase_z1}, \eqref{eq: sec_acc_case_z1} and \eqref{eq: seccase_z0}, the probability of a token $x \in x_\text{acc}$ being accepted under \textit{SPRINTER} is:
    \begin{align}
         &=q(x) \times (1-\eta_\text{FN}+ \eta_\text{FN}) + (1-\eta_\text{FP}) \times (p(x)-q(x)) \nonumber\\
         & = q(x) +(1-\eta_\text{FP}) \times (p(x)-q(x))  \nonumber\\
         & = (1-\eta_\text{FP}) p(x) + \eta_\text{FP} q(x). \label{eq: theorem1_side2}
    \end{align}
    Together with \eqref{eq: theorem1_side1} and \eqref{eq: theorem1_side2}, we complete the proof of Theorem \ref{the:SPRINTER}.
    
 Furthermore, the total variation distance between $p_\text{SPRINTER}$ and $p$ can be calculated as follows:
 \begin{align}
 d_\text{TV}(p,p_\textit{SPRINTER}) &= \frac{1}{2}\sum_{x} |p(x) - p_\textit{SPRINTER}(x)| \nonumber\\
 &= \frac{1}{2}\sum_{x} |p(x) - (1-\eta_\text{FP}) p(x) - \eta_\text{FP}q(x)| \nonumber\\
  &= \frac{1}{2}\sum_{x} |\eta_\text{FP}p(x) -\eta_\text{FP}q(x)|\nonumber\\
    &= \eta_\text{FP} \frac{1}{2}\sum_{x} |p(x) -q(x)|\nonumber\\
    &= \eta_\text{FP} d_\text{TV}(p,q).
 \end{align}
\end{proof}

\subsection{Proof of Theorem \ref{the:SPRINTER_expected_tokens} }
\label{sec:expec_proof}
\begingroup
\renewcommand{\thetheorem}{\ref{the:SPRINTER_expected_tokens}} 
\begin{theorem}
 The probability distribution of the number of generated tokens is given as: 
    \[
\mathbb{P}(N_\text{SPRINTER}=i) =
\begin{cases}
\eta_\text{TP}^{i} (1-\eta_\text{TP}) & i < r, \\
\eta_\text{TP}^r(\eta_\text{FP})^{i-r}(1-\eta_\text{FP}) & i\geq r.
\end{cases}
\]
The expected number of generated tokens is given as:
    \begin{align}
\label{eq:SPRINT_EXP_TOKEN}
        \mathbb{E}(N_\text{SPRINTER})=       \frac{\eta_\text{TP}-\eta_\text{TP}^r}{1-\eta_\text{TP}} + \frac{\eta_\text{TP}^r}{1-\eta_\text{FP}}.
        \end{align}
\end{theorem}
\endgroup


\begin{proof}  We assume that the first $r$ tokens are acceptable, while subsequent tokens are unacceptable. This can be modeled as a finite geometric series for the accepted tokens and an infinite geometric series for the rejected ones. The expected number of generated tokens is derived as follows:
 \begin{align}
     \mathbb{E}[N_\textit{SPRINTER}] &=  \sum_{k=0}^\infty kP(N_\textit{SPRINTER}=k)\nonumber\\
     &= \underbrace{\sum_{k=0}^{r-1}kP(N_\textit{SPRINTER}=k)}_{\text{Term 1}} +\underbrace{\sum_{k=r}^{\infty}kP(N_\text{SPRINTER}=k)}_\text{Term 2}     
 \end{align}
 We first expand Term 1 as follows:
    \begin{align}     
          \sum_{k=0}^{r-1}kP(N_\textit{SPRINTER}=k) &= 0\times(1-\eta_\text{TP}) + 1\times(1-\eta_\text{TP})\eta_\text{TP} + 2\times(1-\eta_\text{TP})\eta_\text{TP}^2 + \ldots + (r-1)\times (1-\eta_\text{TP})\eta_\text{TP}^{r-1}\nonumber\\
         &\stackrel{(a)}= (1-\eta_\text{TP})\sum_{k=0}^{r-1}k\eta_\text{TP}^k\nonumber
    \end{align}
    Note that (a) represents a finite geometric series, allowing us to apply the sum formula for such a series, resulting in:
    \begin{align}
        (1-\eta_\text{TP})\sum_{k=0}^{r-1}k\eta_\text{TP}^k& =  (1-\eta_\text{TP})\eta_\text{TP}\frac{1-r\eta_\text{TP}^{r-1}+(r-1)\eta_\text{TP}^{r}}{(1-\eta_\text{TP})^2} \nonumber\\
         & = \frac{\eta_\text{TP}-\eta_\text{TP}^{r}}{1-\eta_\text{TP}} -(r-1)\eta_\text{TP}^{r}
         \label{eq: theorem2_side1}
    \end{align}
         
Similarly, we simplify Term 2 as follows:
    \begin{align}  
  \sum_{k=r}^{\infty}kP(N_\textit{SPRINTER}=k)
          &\stackrel{(a)}=  r\times\eta_\text{TP}^r(1-\eta_\text{FP}) + (r+1)\times\eta_\text{TP}^r \eta_\text{FP}(1-\eta_\text{FP})+ (r+2)\times\eta_\text{TP}^r \eta_\text{FP}^2(1-\eta_\text{FP}) + \ldots  \nonumber\\
         &\stackrel{(b)}= ar + a(r+1)\eta_\text{FP} +  a(r+2)\eta_\text{FP}^2 +\ldots \nonumber\\
         &= a\sum_{k=r}^{\infty} k\eta_\text{FP}^{k-r} \nonumber\\
         &= \frac{a}{\eta_\text{FP}^r}\sum_{k=r}^{\infty} k\eta_\text{FP}^{k} \nonumber\\
         &\stackrel{(c)}= \frac{a}{\eta_\text{FP}^r}\sum_{k=0}^{\infty}k\eta_\text{FP}^k - \frac{a}{\eta_\text{FP}^r} \sum_{k=0}^{r-1}k\eta_\text{FP}^k \label{eq: theorem2_side2_intermediate}
    \end{align}
    where (a) the sum starts from the $r{th}$ token (i.e. the first unacceptable token), (b) follows from setting $a = \eta_\text{TP}^r(1-\eta_\text{FP})$. We also observe that (c) is a combination of two geometric series. Therefore, the first term in \eqref{eq: theorem2_side2_intermediate} can be simplified as:
    \begin{align}
        \frac{a}{\eta_\text{FP}^r}\sum_{k=0}^{\infty}k\eta_\text{FP}^k &= \frac{a}{\eta_\text{FP}^r}\times \frac{\eta_\text{FP}}{(1-\eta_\text{FP})^2} \nonumber\\
        & = \frac{\eta_\text{TP}^r}{\eta_\text{FP}^r}\times \frac{\eta_\text{FP}}{1-\eta_\text{FP}} \label{eq: theorem2_side2_intermediate_1} 
    \end{align}
    Similarly, we simplify the second term in \eqref{eq: theorem2_side2_intermediate} as:
    \begin{align}
        \frac{a}{\eta_\text{FP}^r} \sum_{k=0}^{r-1}k\eta_\text{FP}^k = \frac{\eta_\text{TP}^r}{\eta_\text{FP}^r} \times\left(\frac{\eta_\text{FP}-\eta_\text{FP}^{r}}{1-\eta_\text{FP}} -(r-1)\eta_\text{FP}^{r}\right) \label{eq: eq: theorem2_side2_intermediate_2}
    \end{align}
    Plugging \eqref{eq: theorem2_side2_intermediate_1} and \eqref{eq: eq: theorem2_side2_intermediate_2} into \eqref{eq: theorem2_side2_intermediate}, we have:
    \begin{align}
        \frac{a}{\eta_\text{FP}^r}\sum_{k=0}^{\infty}k\eta_\text{FP}^k - \frac{a}{\eta_\text{FP}^r} \sum_{k=0}^{r-1}k\eta_\text{FP}^k &= \frac{\eta_\text{TP}^r}{\eta_\text{FP}^r}\left(\frac{\eta_\text{FP}}{1-\eta_\text{FP}} - \frac{\eta_\text{FP}-\eta_\text{FP}^{r}}{1-\eta_\text{FP}} + (r-1)\eta_\text{FP}^{r}\right) \nonumber\\
        &= \frac{\eta_\text{TP}^r}{\eta_\text{FP}^r}\left( \frac{\eta_\text{FP}^{r}}{1-\eta_\text{FP}} + (r-1)\eta_\text{FP}^{r}\right)\nonumber\\
        &= \eta_\text{TP}^r \left( \frac{1}{1-\eta_\text{FP}} + (r-1)\right) \label{eq: theorem2_side2}
    \end{align}
    %
    
    Combining \eqref{eq: theorem2_side1} and \eqref{eq: theorem2_side2}, we obtain:
    \begin{align}
         \mathbb{E}(N_\text{SPRINTER}) &= \frac{\eta_\text{TP}-\eta_\text{TP}^{r}}{1-\eta_\text{TP}} -(r-1)\eta_\text{TP}^{r} +  \eta_\text{TP}^r \left( \frac{1}{1-\eta_\text{FP}} + (r-1)\right) \nonumber\\
         &= \frac{\eta_\text{TP}-\eta_\text{TP}^r}{1-\eta_\text{TP}} + \frac{\eta_\text{TP}^r}{1-\eta_\text{FP}}.
    \end{align}
    The above expression gives Theorem \ref{the:SPRINTER_expected_tokens}.
\end{proof}
    \subsection{Proof of Theorem \ref{the:SPRINTER_expected_latency} }
\begingroup
\renewcommand{\thetheorem}{\ref{the:SPRINTER_expected_latency}} 
\label{sec:lat_proof}
\begin{theorem}
 The expected stopping time is given as:
    \begin{align}
        \label{eq:SPRINT_EXP_LATENCY}
        & \mathbb{E}[T_\text{Stop}] =  \frac{(1-\eta_{\text{TP}}^r) t_d}{1-\eta_\text{TP}} + \frac{\eta_\text{TP}^rt_d}{1-\eta_\text{FP}}.
    \end{align} 
\end{theorem}
\endgroup

\begin{proof}
    Similar as above Theorem, we assume that the first $r$ tokens are acceptable, while all subsequent tokens are not. We also note that the verifier's runtime is negligible compared to that of the draft LLM \( M_q \). Therefore, we consider only \( t_d \) for the expected stopping time. We first reduce the expected stopping time as follows:
     \begin{align}
     \mathbb{E}[T_\text{Stop}] &=  \sum_{k=0}^\infty t_d(k+1)P(T_\text{Stop}=k)\nonumber\\
     &= \underbrace{\sum_{k=0}^{r-1}t_d(k+1)P(T_\text{Stop}=k)}_{\text{Term 1'}} +\underbrace{\sum_{k=r}^{\infty}t_d(k+1)P(T_\text{Stop}=k)}_\text{Term 2'}     
 \end{align}

    We now write out Term 1', representing the case that we stop before the $r$th acceptable token, as follows:
    \begin{align}
    \sum_{k=0}^{r-1} t_d(k+1)P(T_\text{Stop}=k) &= t_d\times(1-\eta_\text{TP}) + 2t_d\times(1-\eta_\text{TP})\eta_\text{TP} +  \ldots + rt_d\times (1-\eta_\text{TP})\eta_\text{TP}^{r-1}\nonumber\\
        &=\sum_{k=0}^{r-1}(1-\eta_\text{TP})\eta_\text{TP}^{k}\times(k+1)t_d \nonumber\\
        & =(1-\eta_\text{TP})t_d\sum_{k=0}^{r-1}k\eta_\text{TP}^{k} + (1-\eta_\text{TP})t_d\sum_{k=0}^{r-1}\eta_\text{TP}^{k}, \label{eq:T_stop_term1}
    \end{align}
    where \eqref{eq:T_stop_term1} is the combination of two geometric series. Therefore, \eqref{eq:T_stop_term1} can be simplified as follows:
    \begin{align}
        &(1-\eta_\text{TP})t_d\sum_{k=0}^{r-1}k\eta_\text{TP}^{k} + (1-\eta_\text{TP})t_d\sum_{k=0}^{r-1}\eta_\text{TP}^{k} \nonumber\\
        &= t_d \left(\frac{\eta_\text{TP}-\eta_\text{TP}^{r}}{1-\eta_\text{TP}} -(r-1)\eta_\text{TP}^{r}\right) + t_d(1-\eta_\text{TP}^r) \nonumber\\
        &=t_d\times\left(\frac{\eta_\text{TP}-\eta_\text{TP}^{r}}{1-\eta_\text{TP}}-r\eta_\text{TP}^{r}+1\right) \label{eq: theorem3_side1}
    \end{align}
    We next expand Term 2', which represents the case that we stop at or after the $r$th acceptable token, as follows:
    \begin{align}
        \sum_{k=r}^{\infty}t_d(k+1)P(T_\text{Stop}=k)&=(r+1)t_d\times\eta_\text{TP}^r(1-\eta_\text{FP}) + (r+2)t_d\times\eta_\text{TP}^r \eta_\text{FP}(1-\eta_\text{FP}) + \ldots  \nonumber\\
        &\stackrel{(a)}= b(r+1) + b(r+2)\eta_\text{FP} +  \ldots \nonumber\\
        &= b\sum_{k=r}^{\infty} k\eta_\text{FP}^{k-r} + b\sum_{k=r}^{\infty} \eta_\text{FP}^{k-r}\nonumber\\
         &= \frac{b}{\eta_\text{FP}^r}\sum_{k=r}^{\infty} k\eta_\text{FP}^{k}+ \frac{b}{\eta_\text{FP}^r}\sum_{k=r}^{\infty} \eta_\text{FP}^{k} \label{eq: theorem3_side2_intermediate}
    \end{align}
    where (a) follows from setting $b = \eta_\text{TP}^r(1-\eta_\text{FP})t_d$.
    Note that the first term in \eqref{eq: theorem3_side2_intermediate} can be written as:
    \begin{align}
        \frac{b}{\eta_\text{FP}^r}\sum_{k=r}^{\infty}k\eta_\text{FP}^k &= \frac{b}{\eta_\text{FP}^r}\left(\sum_{k=0}^{\infty}k\eta_\text{FP}^k -\sum_{k=0}^{r-1}k\eta_\text{FP}^k \right) \nonumber\\
        & = \frac{\eta_\text{TP}^r}{\eta_\text{FP}^r} t_d\times\left(\frac{\eta_\text{FP}}{1-\eta_\text{FP}}-\frac{\eta_\text{FP}-\eta_\text{FP}^{r}}{1-\eta_\text{FP}} +(r-1)\eta_\text{FP}^{r}\right) \label{eq: theorem3_side2_intermediate_1}
    \end{align}
    Similarly, the second term in \eqref{eq: theorem3_side2_intermediate} can be expressed as:
    \begin{align}
        \frac{b}{\eta_\text{FP}^r}\sum_{k=r}^{\infty} \eta_\text{FP}^{k} &= \frac{b}{\eta_\text{FP}^r}\left(\sum_{k=0}^{\infty}\eta_\text{FP}^k -\sum_{k=0}^{r-1}\eta_\text{FP}^k\right)  \nonumber\\
        &=\frac{\eta_\text{TP}^r}{\eta_\text{FP}^r} t_d \left(1- (1-\eta_\text{FP}^r)\right) \label{eq: theorem3_side2_intermediate_2}
    \end{align}
    Plugging \eqref{eq: theorem3_side2_intermediate_1} and \eqref{eq: theorem3_side2_intermediate_2} into \eqref{eq: theorem3_side2_intermediate}, we obtain the expected stopping time after or at $r$th acceptable token is:
    \begin{align}
        \frac{\eta_\text{TP}^r}{1-\eta_\text{FP}}t_d+rt_d\eta_\text{TP}^r \label{theorem3_side2}
    \end{align}
    Combining \eqref{eq: theorem3_side1} and \eqref{theorem3_side2}, the final expression of the expected stopping time is given as:
    \begin{align}
        \mathbb{E}[T_\text{Stop}] =  \frac{(1-\eta_{\text{TP}}^r) t_d}{1-\eta_\text{TP}} + \frac{\eta_\text{TP}^rt_d}{1-\eta_\text{FP}}.
    \end{align}
    This completes the proof.
\end{proof}
\subsection{Additional theoretical results}
\label{sec:Acc_rate}
In this section, building on previous work by \cite{leviathan2023fast}, we examine the acceptance rate of a single run of SPRINTER. This analysis provides insight into how the verifiers' statistical properties influence the acceptance rate.
\textbf{Calculation of }$\alpha_{\textit{SPRINTER}}$ \textbf{and} $\beta_{\textit{SPRINTER}}$. It is well known \cite{leviathan2023fast} that the acceptance rate of standard speculative decoding $\beta$ is given as:
\begin{align}
    \beta = 1-d_{TV}(p,q),
\end{align}
where $d_{TV}(p,q)$ is the total variation distance  between the distributions $p$ and $q$. In the next Theorem, we derive an analogous result showing the acceptance rate of tokens generated by $M_q$ when using \textit{SPRINTER}.
\begin{theorem}\label{the: acceptance_rate}
    The acceptance rate of \textit{SPRINTER} $\beta_{\textit{SPRINTER}}$ is 
    \begin{align}
        1-(1-\eta_\text{FP})\cdot d_{TV}(p,q).
    \end{align}
\end{theorem}

\begin{proof}
From the \textit{SPRINTER} sampling procedure, we know that $\beta_{\textit{SPRINTER}}$ satisfies the following:
\begin{align}
    \beta_{\textit{SPRINTER}} = \mathbb{E}_{x\sim q(x)}
    \begin{cases} 
1 & \text{if } q(x) \leq p(x) \\
\eta_\text{FP}+\frac{p(x)}{q(x)} (1-\eta_\text{FP})& \text{if } q(x) > p(x).
\end{cases} 
\end{align}
The above expectation can be simplified and computed explicitly as follows:
\begin{align}
   \beta_{\textit{SPRINTER}} & = \sum_{x} \min(q(x), \eta_\text{FP} q(x) + (1-\eta_\text{FP})p(x))\nonumber\\
    &=\sum_{x} \min(\eta_\text{FP} q(x)+(1-\eta_\text{FP})q(x), \eta_\text{FP} q(x) + (1-\eta_\text{FP})p(x))\nonumber\\
    &=\sum_{x} (1-\eta_\text{FP})\min(q(x), p(x))+ \eta_\text{FP} q(x)\nonumber\\
    & = (1-\eta_\text{FP}) \cdot(1-d_{TV}(p,q)) + \eta_\text{FP}\nonumber\\
    &= 1-(1-\eta_\text{FP})\cdot d_{TV}(p,q).
\end{align}
\end{proof}
\begin{remark}
    Note that when $\eta_\text{FP} = 0$, the acceptance rate is the same as that of speculative decoding.
\end{remark}
\begin{remark}
    When $\eta_\text{FP} \neq 0$,  $\beta_{\textit{SPRINTER}} \geq \beta$ meaning that \textit{SPRINTER} has a higher acceptance rate compared to that of standard speculative decoding. The inherent reason is that the classifier is susceptible to predicting that the token is acceptable when it actually should be rejected, causing a higher acceptance rate and subsequently enabling the draft model to generate more tokens. However, $\eta_\text{FP}$ also reflects the distance of the probability distribution of \textit{SPRINTER} from the probability distribution of the target model $p$. This means that while more tokens are being accepted, the distribution they are sampled from may deviate from the distribution of $M_p$, causing poorer quality tokens to be generated. Therefore, $\eta_\text{FP}$ measures the tradeoff between the accuracy of the tokens generated and the latency incurred from using \textit{SPRINTER}.
\end{remark}

\subsection{Full SPRINTER Algorithm}
\label{sec:algorithm_hyper}
Algorithm \ref{alg:SPRINTER} provides the detailed overview of Algorithm stated in the main paper, detailing the procedure of a full step of SPRINTER. 
 \begin{algorithm}[h]
   \caption{SPRINTER (detailed)}
   \label{alg:SPRINTER}
\begin{algorithmic}
   \STATE {\bfseries Input:} $M_p, M_q, V, \text{Prefix}, \text{Prediction Threshold}~\tau$
   \STATE $x = \emptyset$
   \STATE \textcolor{blue}{Sample tokens from $M_q$ while verifier predicts $\frac{q(x)}{p(x)}\leq 1$ }
   \STATE $r = 1$
   \WHILE{True}
      \STATE Prefix = Prefix + $x$
      \STATE $q_r(x) \sim M_q(\text{Prefix})$
    \STATE $x \sim q_r(x)$
    \IF{$V (s(x, \text{Prefix})) \geq \tau$}
    \STATE $r += 1$
    \ELSE
    \STATE Break
    \ENDIF
    \ENDWHILE
   \STATE \textcolor{blue}{Run $M_p$ Once}.
   \STATE $p_{r+1}(x) \sim M_p(\text{Prefix}+x_{r})$
   \STATE $p'(x) \leftarrow p_{r+1}(x)$
   \STATE $c \sim U(0,1)$
   \STATE $n = \min(1, \frac{p_{r+1}(x)}{q_{r+1}(x)})$
   \IF{$c > n$}
   \STATE $p'(x) \sim \text{norm}(\max(0, p_{r+1}(x)- q_{r+1}(x))$
   \ENDIF
  \STATE $t \sim p'(x)$
  \STATE \textbf{Return} \text{Prefix} + $t$ 
\end{algorithmic}
\end{algorithm}

\textbf{More Training Details Regarding the Verifier and Hyperparameter Tuning}
During the training process of the verifier, we vary the value of \(\lambda\) among 1, 1.2, and 1.5 to generate the ground truth regarding whether a token should be accepted or rejected. By increasing $\lambda$, we bias the verifier to accept more draft tokens, potentially increasing $\eta_\text{FP}$. Experimental results show that SPRINTER with \(\lambda = 1.2\) performs best on the LM1B dataset for both GPT-Neo and GPT2 model pairs. On the Wiki-Summary dataset, SPRINTER with \(\lambda = 1\)  and \(\lambda = 1.2\) achieves superior results using the GPT2 and GPT-Neo model pairs respectively.  
For the verifier's prediction threshold \(\tau\), we observe that \(\tau = 0.5\) is sufficient to achieve strong performance. 

We split the dataset into train/test sets, using early stopping to mitigate overfitting. The input dimension of the verifier depends on the dimension of the assumed draft model. As GPT-Neo-125M has an output dimension of $768$, for example, the verifier used would also have a dimension of 768 neurons.

\subsection{Detailed Discussion of Related Works}
\label{sec:relatedworks_append}

Multiple approaches for optimizing SD have been presented in the literature. SpecTr \cite{sun2024spectr} uses optimal transport to effectively manage the draft selection process, ensuring efficient and accurate sampling from large models. Yin et al. \cite{yin2024theoretical} frame the decoding problem through a Markov chain abstraction and analyzing its key properties—output quality and inference acceleration—from a theoretical perspective. Recently, Judge Decoding \cite{judge} uses a trained linear head to judge token validity/quality beyond mere alignment, significantly speeding up inference while maintaining output quality. Under this method, tokens that may not necessarily match the target distribution may still be accepted as they are still of high-quality. Similarly, \cite{jang2024lanternacceleratingvisualautoregressive} presents a variant of speculative decoding for visual autoregressive models by relaxing the constraint that tokens must match the target distribution while proposing a mechanism based on total variation distance to ensure that the distribution drift between the generated tokens and the target model does not exceed a certain threshold. 

There are also works that have similarly used an additional classifier for determining how many tokens the draft model should generate. \cite{mamou2024dynamicspeculationlookaheadaccelerates} proposed using a two layer feedforward network to predict when the draft model should stop generating tokens and initiate the target model's verification procedure. \cite{huang2024specdec++} models the SD procedure as a Markov Decision Process and uses the draft model with an added head to predict if the draft model should stop generating tokens. However, both methods ensure that sampling the draft model with their method is equivalent to sampling the target model, which hinders improved latency reductions that can be attained if this requirement is relaxed.  \cite{agrawal2024adaedl} proposes prove that a function of the entropy of the distribution of the draft model can be used as a means of indicating when to end a round of drafting tokens. Specifically, the output of this function is compared against a threshold, which is also adjusted dynamically based on the current acceptance rate, to determine if the current drafting round should end. \cite{kim2024speculative} provides two heuristics for determining when the target model should take control in generating tokens. The first involves observing the draft model's distribution to see if it has a certain lack of confidence in its current token, which would indicate that the target model should generate a replacement token. The second takes place during verification, when the distributions of the target model and draft model are compared to observe an instant in which the draft model is too confident in its decision. This indicates that the sequence should revert back to this point, with the target model generating a replacement token. A method for fine-tuning the draft model to generate tokens that better align with the target model is also presented.

\cite{lu2024improving} investigated using different classifier architectures for halting the drafting process for multi-candidate speculative decoding. MEDUSA \cite{cai2024medusa} propose a framework to accelerate inference in large language models (LLMs) by employing multiple decoding heads. By introducing parallelism in the decoding process, MEDUSA aims to improve the efficiency of LLM operations significantly.
EAGLE \cite{li2024eagle} presents a speculative sampling framework that improves large language model (LLM) inference by predicting second-to-top-layer features based on a token sequence advanced by one time step.

\subsection{Additional Examples of Responses generated by Sprinter and SD}
\label{sec:prompts_ex}
In this section, we feed additional prefixes to \textit{SPRINTER} and SD and show their respective generated responses in Figure \ref{fig:Prompt_example_appendix}. We observe that \textit{SPRINTER} can achieve a relatively good performance compared with SD, but can sometimes generate incorrect information. 
\begin{figure}[htbp]
    \centering
    \begin{minipage}[t]{0.48\linewidth}
        \centering
        \includegraphics[width=\linewidth]{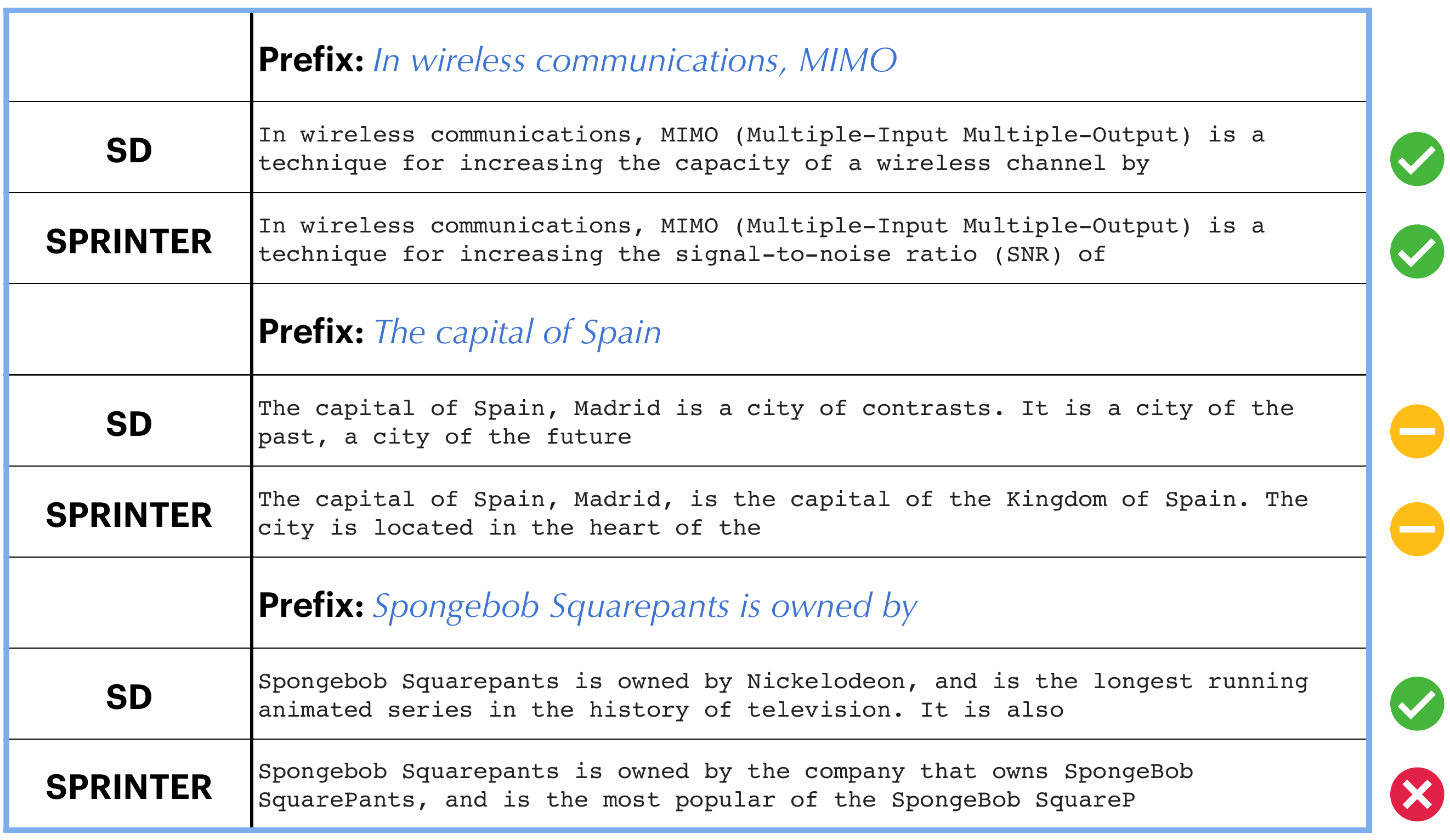}
    \end{minipage}%
    \hfill
    \begin{minipage}[t]{0.48\linewidth}
        \centering
        \includegraphics[width=\linewidth]{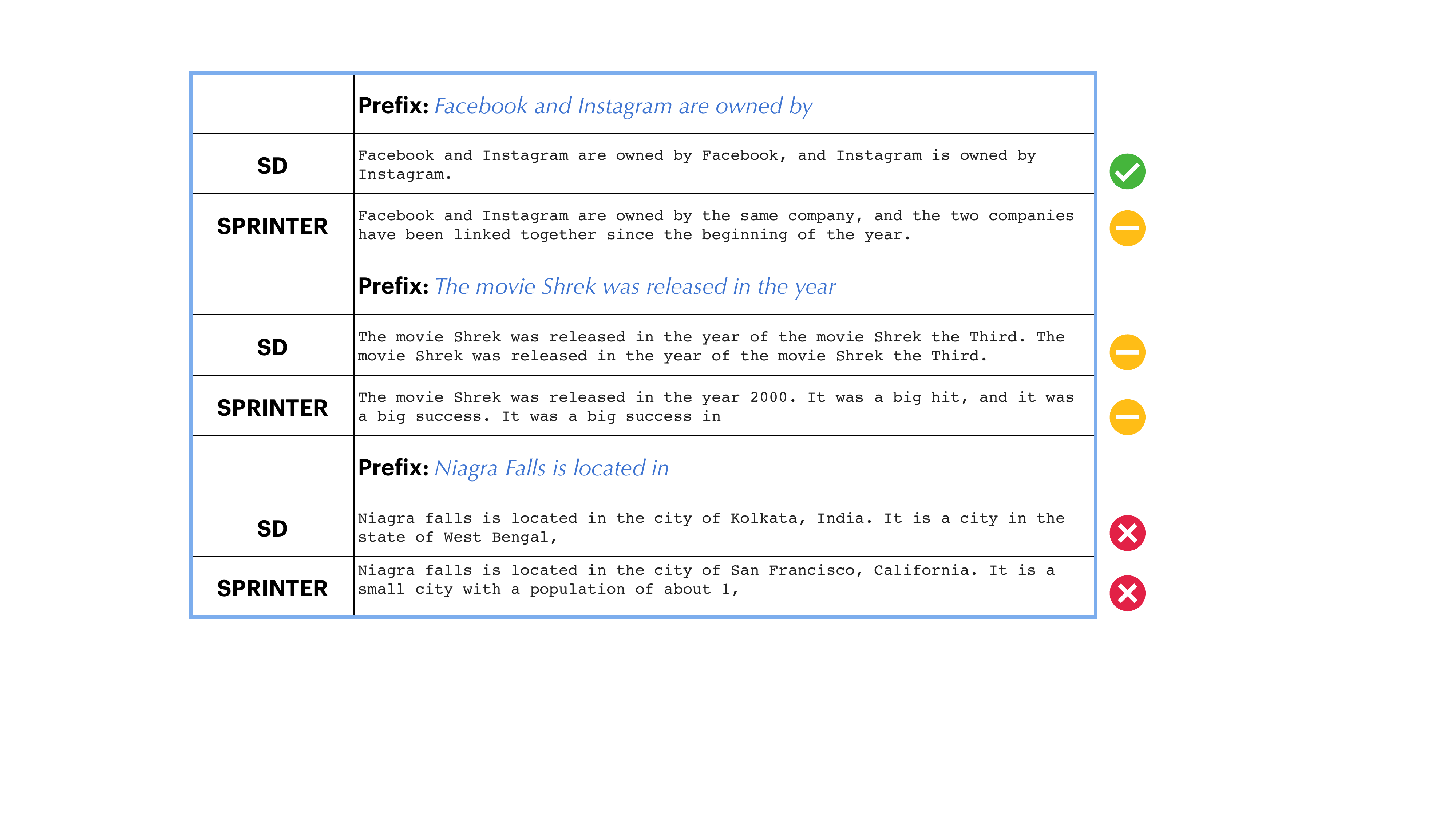}
    \end{minipage}
    \caption{Comparison of responses generated by \textit{SPRINTER} and SD under the same prefixes.}
    \label{fig:Prompt_example_appendix}
\end{figure}

\subsection{Flops Calculation}
\label{sec:flops_append}
We adopt the methods used in \cite{hoffmann2022trainingcomputeoptimallargelanguage,casson2023transformerflops} to determine the number of floating point operations (FLOPS) performed in a forward pass of the draft and target models used in this work. The main sources of FLOPS considered in \cite{hoffmann2022trainingcomputeoptimallargelanguage,casson2023transformerflops} are due to the embedding matrices, self-attention operations in the transformer blocks of the LLM, the feedforward networks in each transfomer block, and the operations required to generate the final logits. Table \ref{tab:Flops} presents the number of FLOPS needed for each draft and target model to generate 20 tokens. The table shows that GPT-Neo-1.3B and GPT2-XL require roughly 8 and 10 times the number of FLOPS compared to GPT-Neo-125M and GPT2-Small respectively. 


\begin{figure}[ht]
    \centering
    \begin{minipage}{0.48\textwidth}
        \centering
        \includegraphics[width=0.7\textwidth]{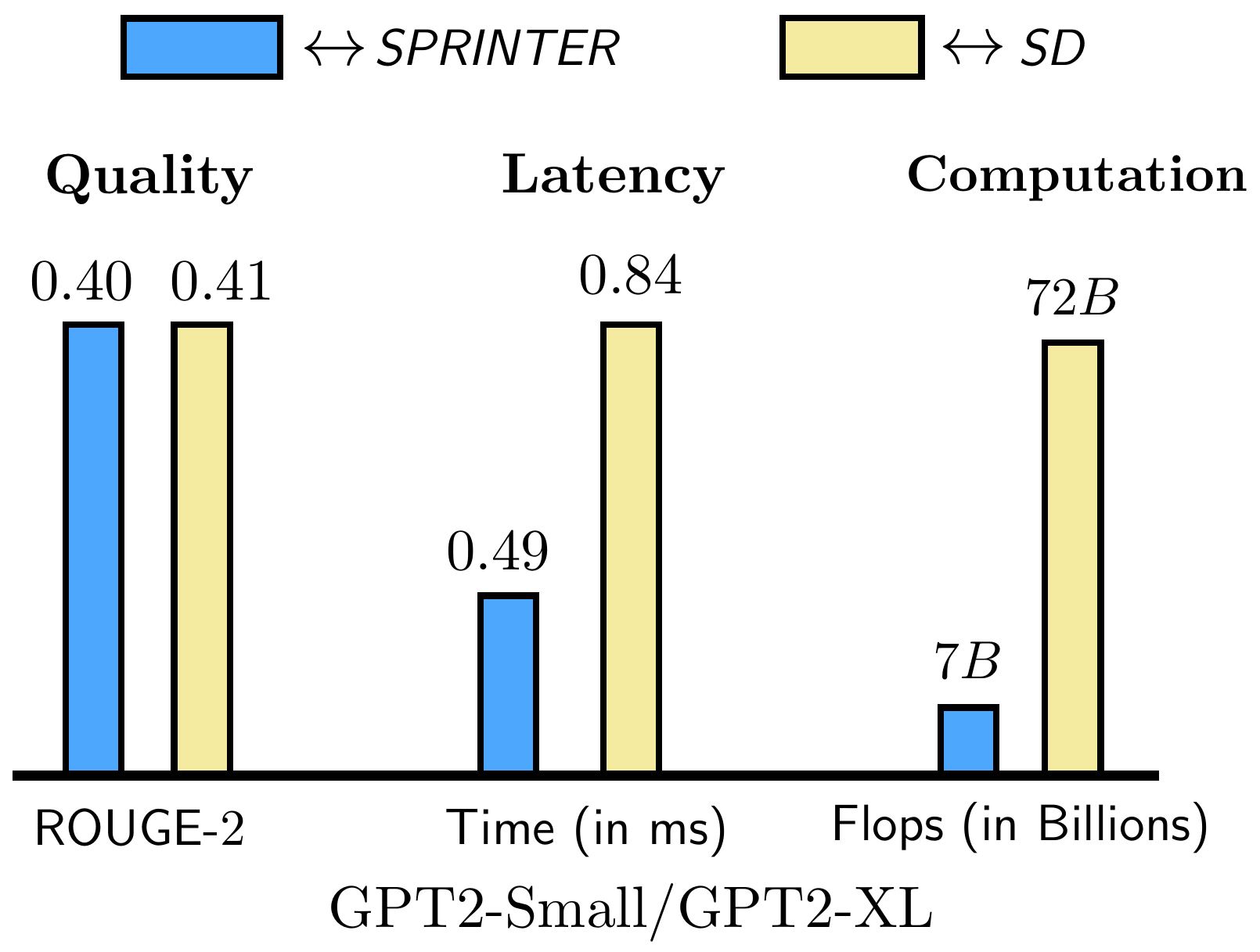}
        \caption{Comparison between \textit{SPRINTER} and SD in terms of \textit{Quality} (ROUGE score), \textit{Latency} (ms to generate a sentence), and \textit{Computation} (FLOPs to generate 20 acceptable tokens) for GPT2-Small and GPT2-XL.}
        \label{fig:quality_latency_comput_appendix}
    \end{minipage}%
    \hfill
    \begin{minipage}{0.48\textwidth}
        \centering
        \begin{tabular}{|c|c|}
            \hline
            Model & FLOPs for 20 tokens \\
            \hline
            GPT-Neo-125M & $8.01$B \\
            GPT-Neo-1.3B & $64.66$B \\
            GPT2-Small & $7.25$B \\
            GPT2-XL & $71.78$B \\
            \hline
        \end{tabular}
        \caption{Estimated FLOPs required for each model to generate 20 tokens.}
        \label{tab:Flops}
    \end{minipage}
\end{figure}


Recall from Section \ref{sec:theory}, that if $\gamma$ consecutive tokens are generated by the draft model, our verifier is of a lower complexity to the draft model, and, as evidenced by Figure \ref{tab:Flops}, that the number of FLOPS used by the target model are significantly greater than the FLOPS used by the target model, then the computational savings experienced under \textit{SPRINTER} is $(\gamma-1)F_t$ where $F_t$ denotes the number of FLOPS of the target model. This further shows that \textit{SPRINTER} is significantly less computationally expensive compared to SD, which relies on the target model for parallel verification every $\gamma$ tokens.

 Figure \ref{fig:quality_latency_comput_appendix} shows the quality-latency-computational profile experienced by \textit{SPRINTER} and SD assuming the GPT2-Small/GPT2-XL model pair. Similar to Figure \ref{fig:frameworkSPrinter}, which shows the same profile for the GPT-Neo model pair, we observe that \textit{SPRINTER} is able to incur a lower latency than SD while suffering a minimal dip in quality and incurring more computational savings by running inference on the draft model more frequently than the target model. 

\subsection{Prompt Design for Win-tie Rate Evaluation}\label{sec:add_quali_results}
\label{sec:add_qual}
In this work, GPT-4 is used to determine the win-tie rates. Specifically, GPT-4 is given the original prefix, a completion generated by \textit{SPRINTER}, and a completion generated by a baseline method. Similar to \cite{chakraborty2024transferqstarprincipled}, GPT-4 is then prompted to evaluate the quality of completions based on accuracy and level of detail of the responses. We provide an example of using GPT-4 to evaluate two prompt responses pairs in Fig \ref{fig:win_rate_exp}.
\begin{figure}
    \centering
    \includegraphics[width=0.85\linewidth]{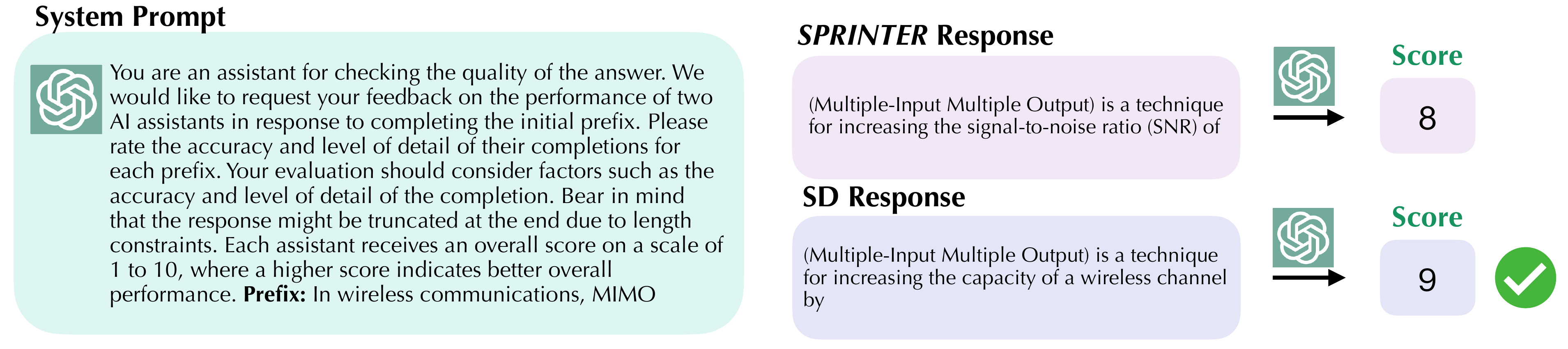}
    \caption{Illustrating win-tie rates evaluation process.}
    \label{fig:win_rate_exp}
\end{figure}
Assume we have the prefix "In wireless communications, MIMO ' and feed this to both \textit{SPRINTER} and SD. The completion from SD is ``(Multiple-Input Multiple Output) is a technique for increasing the capacity of a wireless channel by ". The completion from \textit{SPRINTER} is ``(Multiple-Input Multiple Output) is a technique for increasing the signal-to-noise ratio (SNR) of". GPT-4 is given the prefix and the two completions and gives a score of 8 and 9 to \textit{SPRINTER} and SD respectively, which means that SD is assigned the win for this prefix. 



\subsection{Additional Experimental Results}
\label{sec:add_results}
In this section, we present additional experimental results on the verifier's hyperparameter tuning. Specifically, we show the ROC curve for the Wiki-Summary dataset using the GPT-Neo-125M / GPT-Neo-1.3B model pair as shown in Figure \ref{fig:ROC_wiki_neo} and Table \ref{tab:my_label}. Figure \ref{fig:ROC_wiki_lm1b} shows the ROC Curve for the lm1b dataset. An interesting phenomenon emerges: as we increase the verifier decision threshold 
 $\lambda$, the area under the curve improves, reaching its optimal performance at $\lambda = 1.2$. The optimal balance at $\lambda = 1.2$ suggests that this threshold best separates acceptable from unacceptable tokens.
\begin{figure}[ht]
    \centering
    \begin{minipage}{0.45\textwidth}
        \centering
        \includegraphics[width=\textwidth]{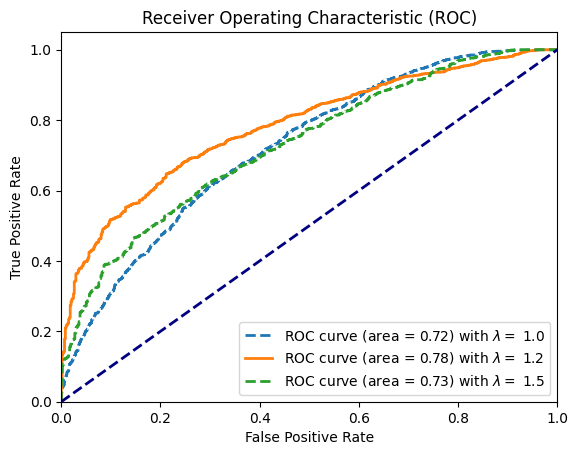}
        \caption{
        ROC curves on the Wiki-Summary dataset with varying acceptance threshold $\lambda$ for the $q(.)/p(.)$ ratio during training. We used $\lambda = 1.2$ for generating latency and quality results with GPT2-S/XL and GPT-Neo draft/target pairs.}
        \label{fig:ROC_wiki_neo}
    \end{minipage}
    \hfill
    \begin{minipage}{0.45\textwidth}
        \centering
        \includegraphics[width=\textwidth]{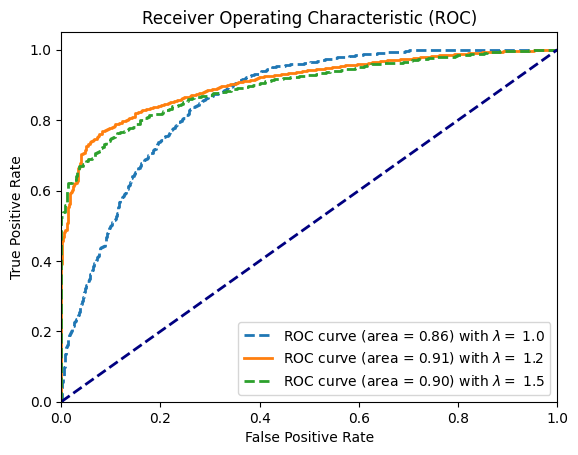}
        \caption{ROC curves on the LM1B dataset with varying acceptance threshold $\lambda$ for the $q(.)/p(.)$ ratio during training. We used $\lambda = 1.2$ for generating latency and quality results with GPT2-S/XL and GPT-Neo draft/target pairs.}
        \label{fig:ROC_wiki_lm1b}
    \end{minipage}
\end{figure}
\begin{table}[htbp]
    \centering
    \begin{tabular}{c|c|c}
    &Wiki-Summary& LM1B\\
    \hline
         GPT-Neo Model pair & ($\lambda, \tau$) = ( 1.2,0.5) & ($\lambda, \tau$) = ( 1.2,0.5)\\
\hline
   GPT2  Model pair &  ($\lambda, \tau$) = ( 1.2,0.5)  &($\lambda, \tau$) = ( 1.2,0.5) \\
    \end{tabular}
    \caption{Comparison of hyperparameters across datasets and model pairs}
    \label{tab:my_label}
\end{table}
\end{document}